\documentclass[]{beingbeyond}
\usepackage{enumitem}
\usepackage[toc,page,header]{appendix}

\usepackage[utf8]{inputenc} 
\usepackage[T1]{fontenc}    
\usepackage{hyperref}       
\usepackage{url}            
\usepackage{array}          
\usepackage{booktabs}       
\usepackage{amsfonts}       
\usepackage{nicefrac}       
\usepackage{microtype}      
\usepackage{xcolor}         
\usepackage{xspace}
\usepackage{bm}
\usepackage{bbm}
\usepackage{tabularx}
\usepackage{amssymb}
\usepackage{enumitem}
\usepackage{amsmath}
\usepackage{mathtools}
\usepackage{amsthm}
\usepackage{multirow}
\usepackage{makecell}
\usepackage{color}
\usepackage{colortbl}
\usepackage{adjustbox}
\usepackage{caption}
\usepackage{graphicx}
\usepackage{pifont}
\usepackage{wrapfig}
\usepackage{mdframed}
\usepackage{multicol}

\newcommand{\method}{FAST\xspace}

\theoremstyle{plain}
\newtheorem{theorem}{Theorem}[section]
\newtheorem{proposition}[theorem]{Proposition}

\theoremstyle{definition}

\theoremstyle{remark}

\title{General Humanoid Whole-Body Control via Pretraining and Fast Adaptation}

\author{{\bfseries Zepeng Wang$^{1,2}$ \enspace Jiangxing Wang$^{2}$ \enspace Shiqing Yao$^{2}$ \enspace Yu Zhang$^{2}$ \enspace Ziluo Ding$^{2}$ \enspace Ming Yang$^{2,3}$ \enspace Yuxuan Wang$^{2,3}$ \enspace Haobin Jiang$^{2,3}$ \enspace Chao Ma$^{1}$ \enspace Xiaochuan Shi$^{1,\dagger}$ \enspace Zongqing Lu$^{2,3,\ddagger}$}}

\affiliation{{$^{1}$Wuhan University \quad $^{2}$BeingBeyond \quad $^{3}$Peking University}}

\webpage{\url{https://research.beingbeyond.com/fast}}

\firstfig[width=0.85\linewidth][\textwidth]
  {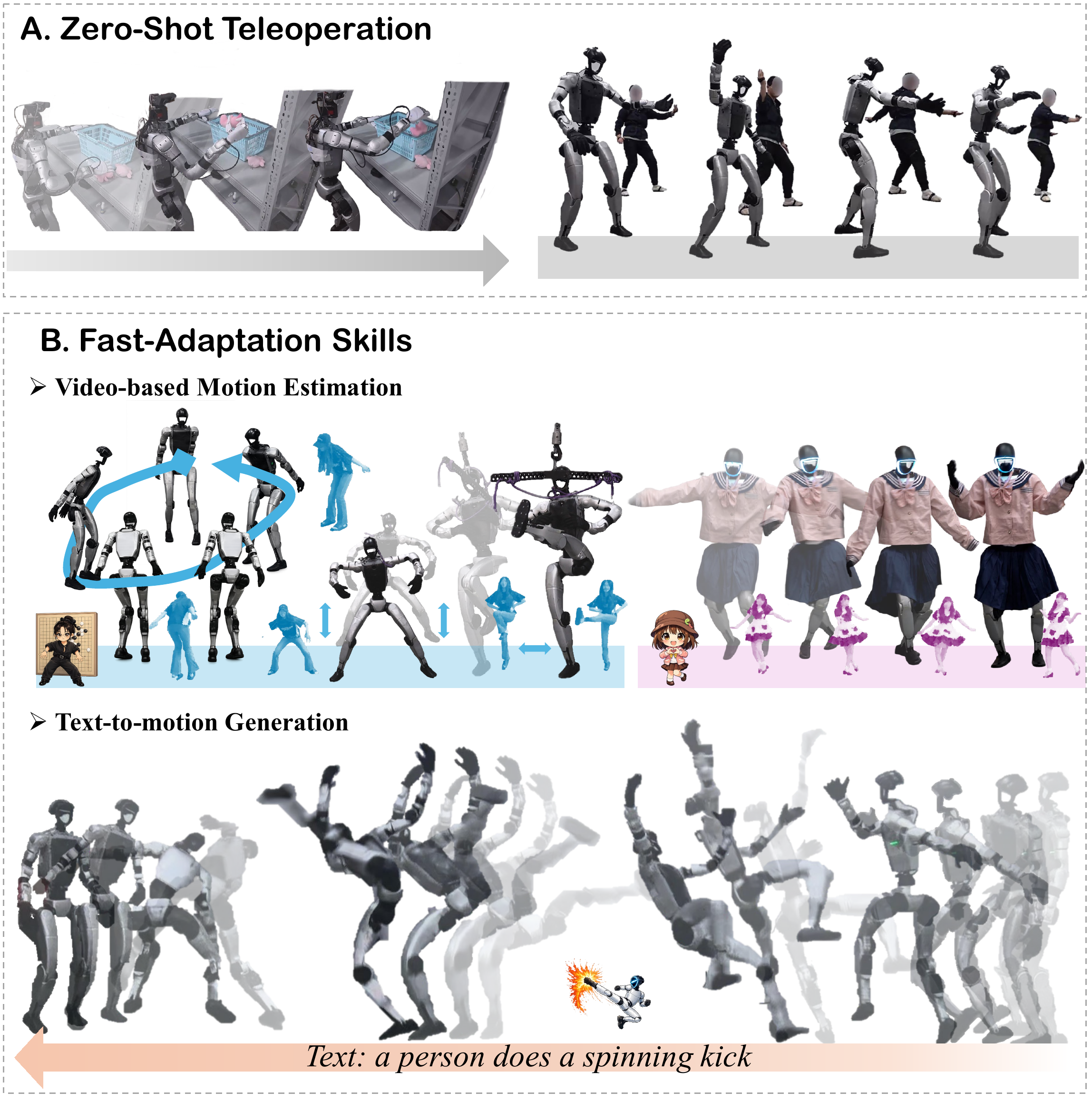}
  {\textbf{FAST} is a \textbf{general} and \textbf{fast-adaptation} framework for humanoid whole-body control. It enables zero-shot high-dynamic motion tracking and real-time teleoperation with strong robustness and balance beyond simulation. For low-quality or out-of-distribution motion references, FAST further supports lightweight residual adaptation for rapid and stable specialization.}
  {fig:main-figure}

\abstract{
Learning a general whole-body controller for humanoid robots remains challenging due to the diversity of motion distributions, the difficulty of fast adaptation, and the need for robust balance in high-dynamic scenarios. Existing approaches often require task-specific training or suffer from performance degradation when adapting to new motions.
In this paper, we present \textbf{FAST}, a general humanoid whole-body control framework that enables \underline{F}ast \underline{A}daptation and \underline{S}table Motion \underline{T}racking. FAST introduces Parseval-Guided Residual Policy Adaptation, which learns a lightweight delta action policy under orthogonality and KL constraints, enabling efficient adaptation to out-of-distribution motions while mitigating catastrophic forgetting. To further improve physical robustness, we propose Center-of-Mass-Aware Control, which incorporates CoM-related observations and objectives to enhance balance when tracking challenging reference motions.
Extensive experiments in simulation and real-world deployment demonstrate that FAST consistently outperforms state-of-the-art baselines in robustness, adaptation efficiency, and generalization. 
}

\checkdata[Date]{February 12, 2026}

\begin{document}

\maketitle

\begingroup
\renewcommand\thefootnote{\fnsymbol{footnote}} 
\setcounter{footnote}{0}
\footnotetext[2]{Correspondence to Xiaochuan Shi $<$shixiaochuan@whu.edu.cn$>$.}
\footnotetext[3]{Project led by Zongqing Lu $<$lu@beingbeyond.com$>$.}
\endgroup

\section{Introduction}
In robotics, achieving \emph{general} whole-body control (WBC) is a critical step toward advancing humanoid intelligence. Rather than mastering isolated skills \cite{he2025asap,xie2025kungfubot}, general WBC aims to extend the boundary of controllable motions and coordination patterns, enabling a single controller to execute diverse whole-body behaviors \cite{luo2025sonic}. Such capability is essential for deploying humanoid robots in realistic settings, including teleoperation for both hazardous operations and large-scale data collection in real-world settings \cite{ze2025twist2}, reliable execution of complex whole-body motions \cite{su2025hitter}, and serving as the low-level motor foundation for embodied intelligence systems \cite{liao2025beyondmimic}.

Despite recent progress in learning-based whole-body control, existing methods remain far from achieving truly \emph{general} motion tracking. Most approaches are trained on curated motion datasets such as AMASS \cite{mahmood2019amass}, OMOMO \cite{li2023object}, and LaFan1 \cite{harvey2020robust}, which are collected under controlled settings and exhibit relatively clean and consistent kinematics. In contrast, real-world motion references often originate from heterogeneous sources, including offline motion capture, real-time teleoperation, text-to-motion generation, and monocular video-based estimation. These sources vary substantially in motion quality, noise, temporal consistency, and physical plausibility, inducing significant distribution shifts \cite{li2025train}.

Under such shifts, existing controllers frequently struggle to robustly track motions that deviate from the training distribution. This issue is particularly pronounced for highly dynamic or stylistically complex motions \cite{pan2025agility,han2025kungfubot2}, where limited policy expressiveness and stability constraints lead to tracking failures even for motions nominally similar to training data. Such motions effectively act as outliers that are difficult to cover with finite datasets \cite{chen2025gmt}. These limitations motivate the need for whole-body controllers that combine strong zero-shot robustness with the ability to rapidly adapt to novel and low-quality motion inputs.

A direct solution is to scale up motion datasets, model capacity, and computational resources to cover a broader range of motions and embodiments \cite{luo2025sonic,he2025viral}. In theory, large-scale pretraining can improve generalization. In practice, however, this paradigm is constrained by real-world control requirements.
Whole-body control demands low-latency, high-frequency inference to maintain closed-loop stability \cite{liao2025beyondmimic}. Humanoid robots typically operate under limited onboard compute and power budgets, making large-scale models difficult to deploy at control frequencies \cite{ahmad2026vision}. Moreover, increased model size introduces additional inference latency, which can directly degrade balance and stability.
Under these constraints, a more practical direction is fast adaptation. Rather than relying solely on exhaustive pretraining, policies should rapidly specialize to new motion distributions through efficient fine-tuning or lightweight adaptation, complementing large-scale training while satisfying deployment and control constraints.

To address robust humanoid whole-body control under distribution shifts, we propose \textbf{FAST}, a general framework that combines a pretrained Center-of-Mass-Aware controller with lightweight residual learning for fast adaptation to unseen and out-of-distribution motions. This design preserves robust zero-shot performance while maintaining physical stability under relatively low-quality reference motions.

Our main contributions are summarized as follows:
\begin{itemize}
\item We propose \textbf{FAST}, a general and fast-adaptation framework for humanoid whole-body control that enables efficient adaptation to out-of-distribution motions while maintaining strong zero-shot robustness.
\item We introduce \textbf{Center-of-Mass-Aware Control}, which explicitly incorporates CoM-related observations and objectives to improve balance and stability when tracking challenging or physically inconsistent motion references.
\item We demonstrate that FAST supports a wide range of whole-body control scenarios, including zero-shot high-dynamic motion tracking, fast adaptation to relatively low quality motions from text and video, and real-time whole-body teleoperation, with consistent gains in robustness and generalization in both simulation and real-world experiments.
\end{itemize}


\section{Related Work}
\subsection{Learning-based Humanoid Whole-Body Control}
Learning-based approaches have achieved significant progress in humanoid whole-body control, enabling a wide range of skills such as locomotion \cite{wang2025crosser,wang2025beamdojo}, agile motions \cite{xie2025kungfubot}, and recovery behaviors \cite{huang2025learning}. However, most methods typically rely on task-specific reward design, which limits scalability and generality. In contrast, whole-body motion tracking leverages retargeted human motion as reference trajectories through kinematic retargeting, eliminating manual reward engineering while promoting coordinated and more generalizable humanoid control.

For learning diverse humanoid motions within a general policy, prior work has explored improvements along several complementary directions \cite{zhang2025track,he2025omnih2o,he2025hover}. Architectural approaches improve scalability by decoupling upper- and lower-body control \cite{ding2025jaeger} or adopting expert-based paradigms to reduce cross-motion interference \cite{wang2025experts}. Data-centric methods emphasize curated large-scale datasets \cite{ji2024exbody2} or augment training with extreme balance recovery to enhance robustness \cite{pan2025agility}. From the perspective of tracking objectives, GMT \cite{chen2025gmt} prioritizes local pose and velocity consistency for stable long-horizon tracking, while KungfuBot2 \cite{han2025kungfubot2} enforces consistency via local style and global stability constraints. More recently, SONIC \cite{luo2025sonic} demonstrates that scaling model capacity and data volume enables foundation-style policies to generalize across broad motion distributions.

In contrast, FAST emphasizes fast and stable adaptation under distribution shift, rather than scaling model size or retraining task-specific policies, enabling robust whole-body control for high-dynamic and out-of-distribution motions.

\subsection{Fast Adaptation for Humanoid Control}
Fast adaptation in humanoid control has been widely studied to bridge distribution gaps between reference motions and physical execution. A common paradigm is residual reinforcement learning, where a post-trained residual policy compensates for mismatches between kinematic references and robot dynamics \cite{johannink2019residual}. Representative methods such as ResMimic \cite{zhao2025resmimic} and ASAP \cite{he2025asap} apply residual learning to refine motion tracking under physical constraints. However, these approaches mainly target sim-to-real dynamics correction and are not designed for rapid adaptation to novel or long-tailed motion distributions.

Another line of work improves adaptation efficiency by leveraging pre-trained priors or skill representations \cite{peng2022ase,tessler2023calm}. Hierarchical frameworks such as SPiRL \cite{pertsch2021accelerating} constrain downstream policy updates via KL regularization to remain close to learned skill priors, while large-scale pre-training approaches like Octo \cite{mees2024octo} demonstrate that fine-tuning pre-trained models significantly outperforms training from scratch. Nevertheless, these methods rely on the transferability and expressiveness of the learned priors, which can limit adaptation when target motions deviate substantially from pre-training distributions.

To address highly diverse and challenging motion datasets, more targeted adaptation mechanisms have been proposed. PHC \cite{luo2023perpetual} incrementally extends policy capacity by adding new motion-specific modules while freezing existing ones to mitigate forgetting. UniTracker \cite{yin2025unitracker} adopts a lightweight residual decoder on top of a frozen universal policy to capture extreme out-of-distribution motions efficiently.

Distinct from these directions, FAST combines residual adaptation with explicit stability regularization, enabling fast adaptation to novel motion distributions while preserving robustness and prior tracking capabilities.

\section{Method}

We present a unified framework, FAST, for general humanoid whole-body control that combines large-scale motion pretraining with stable and efficient policy adaptation, as illustrated in Figure~\ref{fig:framework}. A motion tracking policy is trained to follow diverse reference motions while explicitly regulating balance and stability through Center-of-Mass–Aware objectives and observations. To enable fast specialization to new motion distributions, we introduce a residual adaptation scheme with Parseval regularization and KL-constrained updates, allowing efficient adaptation without degrading previously learned behaviors.

\begin{figure*}[t]
	\centering
        \includegraphics[width=\linewidth]{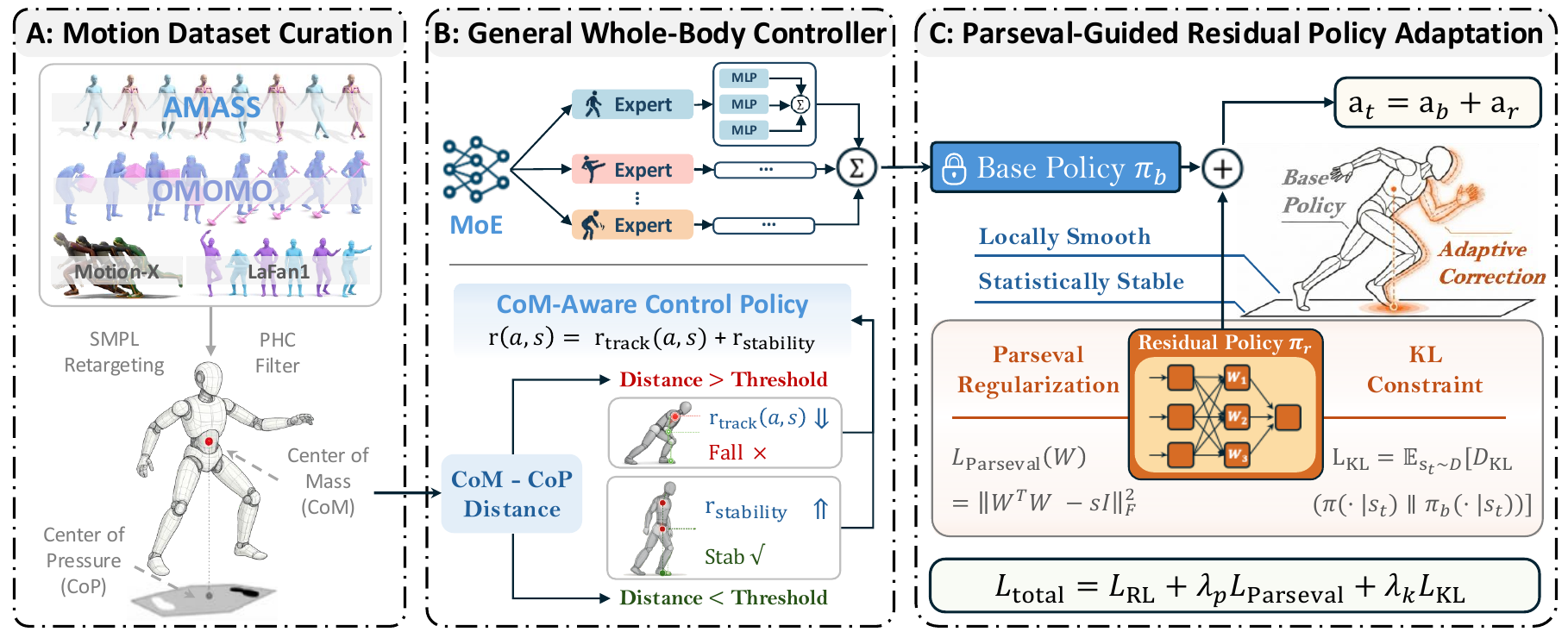}
\caption{\textbf{An overview of FAST.} Our framework consists of three stages.
(1) We construct a curated humanoid motion dataset via human-to-humanoid retargeting with auxiliary physical signals.
(2) We train a general whole-body controller with a Mixture-of-Experts architecture and Center-of-Mass-Aware control.
(3) We perform fast adaptation via Parseval-guided residual policy learning.}
 \label{fig:framework}
\end{figure*}

\subsection{Motion Dataset Curation}
\label{method:retarget}

\noindent \textbf{Human-to-Humanoid Retargeting}. We construct a large-scale humanoid motion dataset by retargeting publicly available human motion capture datasets, including AMASS \cite{mahmood2019amass} and OMOMO \cite{li2023object}, represented in the SMPL format. Additionally, we also collect a small in-house dataset using our own MoCap system. The retargeting and PHC-filtering process follows the previous work \cite{wang2025experts}. These datasets are used mainly for the large-scale pretraining of the general whole-body controller. To evaluate fast adaptation and cross-distribution generalization, we additionally incorporate in-house motion data collected via Inertial Motion Capture, together with LaFan1 \cite{harvey2020robust} and MotionX \cite{lin2023motion}, which contain more dynamic and device-specific motion patterns.

\noindent \textbf{Data Augmentation for Motion Diversity}. To enhance motion diversity and robustness to distribution shifts, we augment retargeted humanoid motions by adjusting global translational velocity and perturbing lower-limb joint configurations. This produces a broader range of motion speeds and kinematic variations while preserving physically coherent dynamics. More details are in Appendix \ref{data_augmentation}.

\noindent \textbf{Auxiliary Physical Signals}. In addition to kinematic data, we augment each motion with contact masks derived from foot–ground interactions, and include the center of mass (CoM) and center of pressure (CoP) computed following prior work \cite{xie2025kungfubot}. These physically grounded signals provide explicit balance and support cues, enabling more stable and physically plausible motion tracking.

\subsection{General Whole-Body Controller}
A general motion tracking policy is trained on the motion dataset via Reinforcement Learning. At each timestep, the policy observation contains the proprioceptive information. In addition, the observation also includes command information from the reference motion at that timestep, consisting of both joint angles and selected keypoint positions, as well as the corresponding velocities. The policy is trained using Proximal Policy Optimization (PPO) \cite{schulman2017proximal} and outputs actions as whole-body target joint positions for joint-level proportional-derivative (PD) controllers. More details about the observation design of our framework are provided in the Appendix \ref{training:obs}.

\noindent \textbf{Policy Architecture}. Following prior work \cite{chen2025gmt,wang2025experts}, we adopt a mixture-of-experts (MoE) policy architecture to handle diverse motion dynamics. The policy consists of multiple expert networks, each implemented as a three-layer MLP, and a gating network that outputs state-dependent mixture weights. At each timestep, the final action is a weighted sum of the expert outputs.

This design is motivated by the fact that different motions exhibit distinct dynamic properties. Training a single policy on highly diverse motions often leads to mutual interference during learning. The MoE structure allows different experts to specialize in different motion dynamics, while the gating network learns to combine them based on the current state and reference command.

\textbf{Center-of-Mass-Aware Control}. In the general whole-body control, especially for teleoperation, prior methods \cite{xie2025kungfubot} rely on curated training pipelines that filter out aggressive or unstable motions, while in real-world deployment, the policy inevitably encounters aggressive and unseen references that cannot be filtered in advance. Instead of enforcing strict tracking, we explicitly design the policy to conservatively follow such motions to ensure safe execution. 

Concretely, we augment the observation with reference Center-of-Mass (CoM) and Center-of-Pressure (CoP) signals, and guide conservative behavior through two reward mechanisms. First, when the reference CoM–CoP distance exceeds a threshold, all tracking-related reward coefficients are progressively down-weighted, allowing the policy to relax tracking under aggressive references. Second, we introduce an explicit stability reward that encourages minimizing the distance between the humanoid’s actual CoM and CoP, promoting balance maintenance (See \ref{training:reward} for details). These CoM-aware designs enable a dynamic trade-off between tracking accuracy and physical stability, improving robustness under challenging motions.

\subsection{Parseval-Guided Residual Policy Adaptation}

\noindent \textbf{Delta Policy Parameterization}. To enable efficient adaptation to new motion distributions while preserving the generality of the pretrained controller, we adopt a residual policy on top of a fixed base model. The base policy $\pi_b$ is a pretrained whole-body controller introduced earlier, for which direct end-to-end fine-tuning is computationally expensive and often inefficient in the low-data adaptation setting.

We therefore introduce a lightweight residual policy $\pi_r$, implemented as a three-layer MLP, which predicts additive action corrections conditioned on the same observations as $\pi_b$. During adaptation, the base policy is fully frozen and only the residual policy is optimized, while the critic is initialized from the pretrained model and further fine-tuned. The final action is given by $a_t = a_t^b + a_t^r$.

\noindent \textbf{Parseval Regularization for Stable Adaptation}. While the residual policy reduces the parameter space for adaptation, unconstrained fine-tuning can still be unstable and prone to suboptimal convergence when adapting to new motion distributions that differ substantially from the pretraining domain. To address these challenges, we incorporate Parseval regularization, which both \textbf{\textit{stabilizes optimization and encourages the network to learn diverse, approximately orthogonal feature directions}}. This allows the residual policy to explore beyond the base policy's original behavior, facilitating escape from local optima in the new domain.

Specifically, we apply Parseval regularization \cite{chung2024parseval} to all linear layers of the residual policy except the final output layer. 
For a weight matrix $W \in \mathbb{R}^{m \times n}$, the regularization term is defined as
\begin{equation}
L_{\text{Parseval}}(W) = \| W^\top W - s I \|_F^2 ,
\end{equation}
where $s>0$ is a scaling factor, $I$ denotes the identity matrix of appropriate dimension, and $\|\cdot\|_F$ is the Frobenius norm. 
This regularizer encourages the columns of $W$ to be approximately orthonormal, thereby implicitly constraining the layer’s spectral norm and limiting extreme scaling along any single feature direction. 
As a result, it promotes smoother gradient propagation and better-conditioned feature representations, enabling both stable training and effective exploration of new feature directions during adaptation.

The residual policy is trained with the combined objective
\begin{equation}
L_{\text{total}} = L_{\text{RL}} + \lambda_p \sum_{l=1}^{n-1} L_{\text{Parseval}}(W_l),
\end{equation}
where \( L_{\text{RL}} \) is the standard reinforcement learning loss, \( W_l \) are the weight matrices of all layers except the last, and \( \lambda_p > 0 \) controls the regularization strength. Empirically, Parseval-regularized residual adaptation leads to more stable optimization and improved sample-efficient learning across new motion distributions, successfully escaping local optima that unconstrained updates may get trapped in.

\noindent \textbf{KL-Constrained Policy Update}. While residual parameterization and Parseval regularization constrain adaptation in parameter and representation spaces, they do not by themselves prevent the adapted model from deviating excessively from the pretrained model. 
This is particularly important in our setting, where adaptation is performed after pretraining and relies only on newly encountered motions, without revisiting the source-domain dataset. As a result, the residual policy receives no supervision on previously learned motions, making explicit constraints on policy drift essential.
To mitigate catastrophic drift and preserve general motion tracking capabilities, we introduce a KL-divergence constraint that anchors the adapted model to the base model.

Specifically, we regularize the combined policy $\pi$, defined by $a_t = a_t^b + a_t^r$, to remain close to the base policy $\pi_b$ in terms of action distributions:
\begin{equation}
L_{\text{KL}} = \mathbb{E}_{s_t \sim D} \big[ D_{\text{KL}}(\pi(\cdot|s_t) \| \pi_b(\cdot|s_t)) \big],
\end{equation}
where $D$ is the state distribution induced by the adapted policy. The KL constraint encourages small, task-specific corrections by implicitly limiting residual action magnitude. As a result, the combined model remains close to the base policy’s output, allowing the residual policy to generalize to states from the original distribution even though it is trained solely on new motion data. Empirically, KL-constrained adaptation preserves performance across both original and new motion distributions.

The KL term is incorporated into the overall objective as a soft constraint:
\begin{equation}
L_{\text{total}} = L_{\text{RL}} + \lambda_p L_{\text{Parseval}} + \lambda_{k} L_{\text{KL}},
\end{equation}
where $\lambda_{k}$ controls the strength of the KL constraint.

\noindent \textbf{Stability Analysis of Residual Policy Adaptation}. To support the stability and effectiveness of the proposed fast adaptation framework, we provide a theoretical analysis of the residual policy adaptation scheme. In particular, we show that Parseval regularization and KL-constrained updates provide complementary constraints on the residual policy: Parseval regularization bounds its functional sensitivity via Lipschitz continuity, while the KL constraint limits its deviation from the base policy. Together, they ensure smooth and stable fast adaptation under limited data.
\begin{proposition}
\label{prop:Lipschitz}
Consider a residual policy $\pi_r:\mathbb{R}^d \rightarrow \mathbb{R}^m$ parameterized by an $L$-layer feed-forward MLP with 1-Lipschitz activation functions $\phi(\cdot)$.
Suppose Parseval regularization is applied such that for each layer $l=1,\ldots,L-1$, the weight matrix satisfies
$
\| W_l^\top W_l - s I \|_2 \le \varepsilon_l,
$
where $s>0$ is a scaling factor and $\varepsilon_l \ge 0$ denotes the approximation error.
Then the residual policy $\pi_r$ is Lipschitz continuous, with Lipschitz constant bounded as
$
\mathrm{Lip}(\pi_r)
\le \prod_{l=1}^{L-1} \sqrt{s + \varepsilon_l}\;\|W_L\|_2.
$
\end{proposition}

\begin{proposition}
\label{prop:stability}

Let the adapted policy be defined as
$\pi(\cdot|s)=\mathcal{N}(\mu_b(s)+\mu_r(s),\,\Sigma)$, where $\mu_b(s)$ is the mean of the base policy ,$\mu_r(s)$ is the residual mean, and $\Sigma$ is a state-independent covariance matrix shared by both policies.
If the KL divergence satisfies
$D_{\text{KL}}(\pi(\cdot|s)\|\pi_b(\cdot|s))\leq\epsilon, \forall s$,
then the residual policy output is bounded as
$\|\mu_r(s)\|_2\leq\sqrt{2\epsilon\lambda_{max}(\Sigma)}$, where $\lambda_{\max}(\Sigma)$ denotes the largest eigenvalue of the covariance matrix.
\end{proposition}

\noindent\textit{Proof.} Full proof details of Proposition~\ref{prop:Lipschitz} and Proposition~\ref{prop:stability} are provided in Appendix \ref{prof:prop}.

Taken together, Proposition~\ref{prop:Lipschitz} and Proposition~\ref{prop:stability} provide complementary guarantees for residual policy adaptation.
Parseval regularization bounds the sensitivity of the residual policy with respect to state perturbations (functional smoothness).
KL constraint bounds the magnitude of the residual policy relative to the base policy (distributional proximity).

As a result, the adapted policy remains \textbf{\textit{locally smooth}} and \textbf{\textit{statistically stable}}, enabling efficient adaptation to out-of-distribution motions while preserving prior knowledge.

\section{Experiments}

In the experimental section, we aim to answer the following key research questions:

\begin{itemize}

\item \textbf{Q1}: How does FAST compare with state-of-the-art humanoid whole-body controllers?

\item \textbf{Q2}: Can FAST adapt efficiently to new motion distributions while retaining prior capabilities? 

\item \textbf{Q3}: Does CoM-Aware Control improve balance and stability in high-dynamic motions? 

\item \textbf{Q4}: How effectively does FAST transfer to real-world whole-body control from diverse motion sources?

\end{itemize}

\subsection{Experiment Setup} \label{exp:dataset}

\noindent \textbf{Baselines}. To evaluate the effectiveness of our approach, we compare against two state-of-the-art universal humanoid motion tracker, GMT \cite{chen2025gmt} and TWIST2 \cite{ze2025twist2}. We employ the officially released models for both baselines and conduct all evaluations in the MuJoCo simulator, which is supported by all methods. 

\noindent \textbf{Metrics}. We use a set of metrics \cite{wang2025experts, zhang2025hub} to comprehensively evaluate the performance of motion tracking. (1) \textbf{Success Rate} ($Succ$, \%): in all experiments, a task is considered a failure when the root height of the humanoid deviates from the reference motion by more than 0.25 m. Additionally, for the objectives of different experiments, we design and apply stricter task failure conditions to better reflect the specific requirements of each task (see Appendix \ref{evaluation_failure}). (2) \textbf{Global Mean Per Keypoint Position Error} ($E_{\mathrm{mpkpe}}$, $m$) measures global keypoint position tracking accuracy. (3) \textbf{Mean Per Joint Position Error} ($E_{\mathrm{mpjpe}}$, $rad$) measures local joint position tracking performance. (4) \textbf{Global Mean Root Linear Velocity Error} ($E_{\mathrm{vel}}$, $m$) measures global root linear velocity tracking accuracy. To evaluate the policy’s performance on high-dynamic balance motions, we further include (5) \textbf{Slippage} ($Slip$, $m/s$), measuring the velocity of the support foot relative to the ground, with higher values indicating greater instability, and (6) \textbf{Mean Center-of-Mass and Center-of-Pressure Distance} ($E_{\mathrm{mpd}}$, $m$), which measures the absolute difference between the robot's CoM and CoP, reflecting the current balance state of the humanoid.

\subsection{Comparison with Existing Methods}

To address \textbf{Q1}, Table~\ref{tab:main-result} compares our base model with GMT and TWIST2 on both the training dataset (AMASS) and an out-of-distribution (OOD) dataset (a leg-dominant subset of MotionX \cite{lin2023motion}). 

\begin{table*}[h]
\centering
\caption{Main results on training dataset (AMASS) and test dataset (MotionX). Results are reported as mean $\pm$ one standard deviation.}
\resizebox{\textwidth}{!}{
\begin{tabular}{lcccc @{\hskip 2em} cccc}
\toprule
& \multicolumn{4}{c}{\textbf{AMASS (training dataset)}} 
& \multicolumn{4}{c}{\textbf{MotionX (test dataset)}} \\
\cmidrule(r){1-5} \cmidrule(r){6-9}
\textbf{Method} 
& $Succ \uparrow$ 
& $E_{\mathrm{mpjpe}} \downarrow$ 
& $E_{\mathrm{mpkpe}} \downarrow$ 
& $E_{\mathrm{vel}} \downarrow$ 
& $Succ \uparrow$ 
& $E_{\mathrm{mpjpe}} \downarrow$ 
& $E_{\mathrm{mpkpe}} \downarrow$ 
& $E_{\mathrm{vel}} \downarrow$ 
\\
\midrule
GMT &
{77.55}  & 
{0.100}\textsubscript{$\pm$0.020}  & 
{0.319}\textsubscript{$\pm$0.257}  & 
{0.226}\textsubscript{$\pm$0.162}  & 
{38.91}  & 
{0.124}\textsubscript{$\pm$0.033}  & 
{0.232}\textsubscript{$\pm$0.173}  & 
\textbf{0.273}\textsubscript{$\pm$0.157}  
\\

TWIST2 &
{72.61}  & 
\textbf{0.084}\textsubscript{$\pm$0.022}  & 
{0.388}\textsubscript{$\pm$0.258}  & 
{0.269}\textsubscript{$\pm$0.163}  & 
{31.82}  & 
\textbf{0.102}\textsubscript{$\pm$0.031}  & 
{0.229}\textsubscript{$\pm$0.182}  & 
{0.311}\textsubscript{$\pm$0.191}  
\\

\textbf{\method{} (ours)}    & 
\textbf{87.14}  & 
{0.132}\textsubscript{$\pm$0.040}  & 
\textbf{0.123}\textsubscript{$\pm$0.079}  & 
\textbf{0.206}\textsubscript{$\pm$0.157}  & 
\textbf{51.46}  & 
{0.143}\textsubscript{$\pm$0.035}  & 
\textbf{0.212}\textsubscript{$\pm$0.136}  & 
{0.359}\textsubscript{$\pm$0.198}   
\\

\bottomrule
\end{tabular}
}
\label{tab:main-result}
\end{table*}

Across both datasets, FAST consistently achieves the highest success rate, highlighting improved long-horizon robustness and generalization. 
Although not always optimal in local joint-angle errors, FAST substantially improves the tracking accuracy of global keypoints, which is more directly correlated with physically consistent and stable motions. This reflects a trade-off between optimizing local accuracy and maintaining globally coherent, long-horizon behavior.

On the OOD dataset, baseline methods report lower MPJPE and velocity errors. However, these metrics are computed over significantly shorter successful segments, as frequent early terminations occur during highly dynamic motions. As a result, per-frame error metrics alone can be misleading when evaluating long-horizon tracking performance.

\subsection{Evaluation of Fast Adaptation with Performance Preservation}

To address \textbf{Q2}, we evaluate the fast adaptation capability of FAST on two target datasets, LaFan1 and MotionX, while assessing performance preservation on the source training dataset (AMASS). We compare against several adaptation strategies, including a frozen base model, training from scratch, full fine-tuning, and a residual policy without regularization. For a fair comparison, all methods utilize the same training setup and optimization steps. All simulation evaluations are conducted in the IsaacLab simulator. Quantitative results on LaFan1 and MotionX, alongside performance preservation on AMASS, are visualized in Figure~\ref{fig:adaptation_radar}.

\begin{figure}[h]
	\centering
        \includegraphics[width=0.49\linewidth]{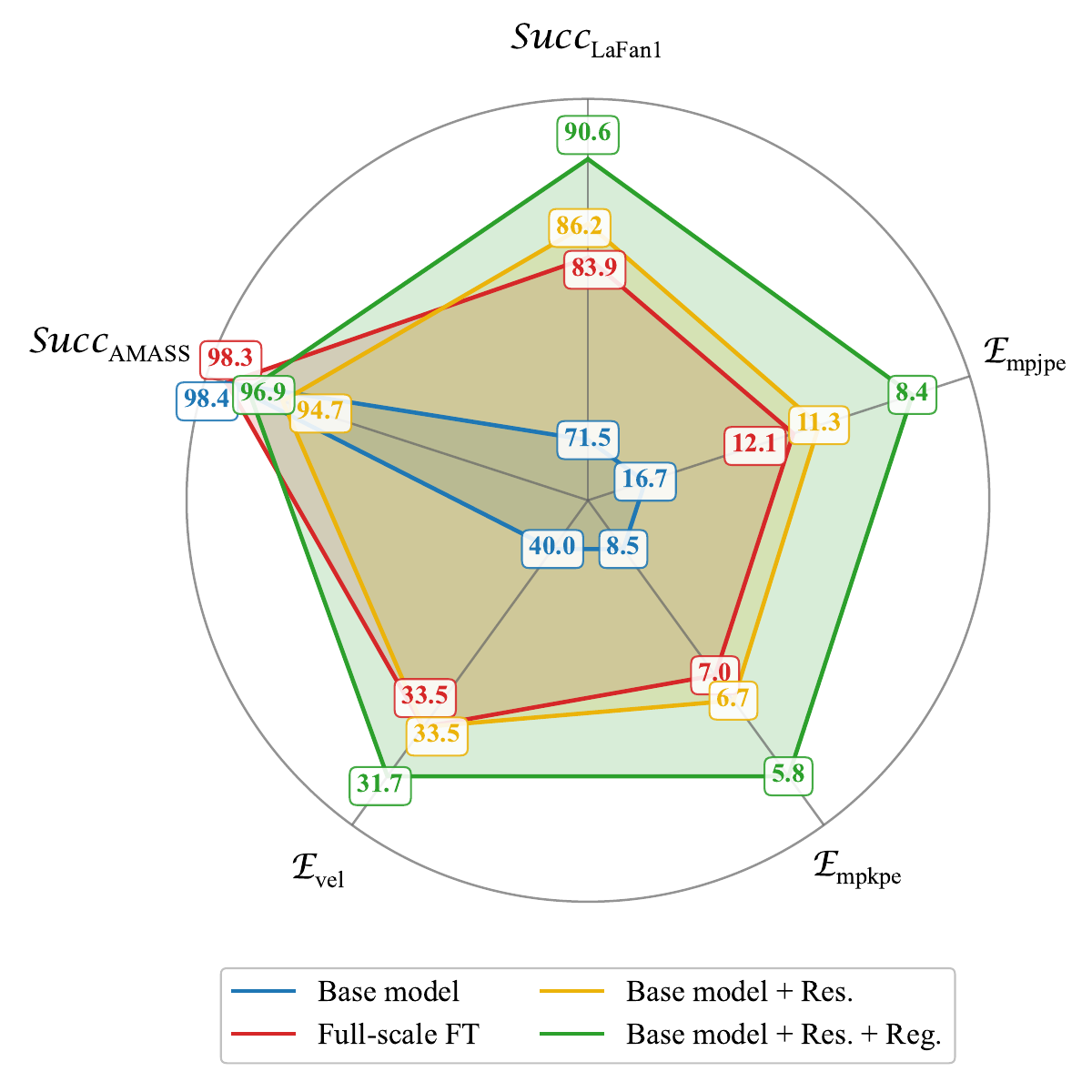}
        \includegraphics[width=0.49\linewidth]{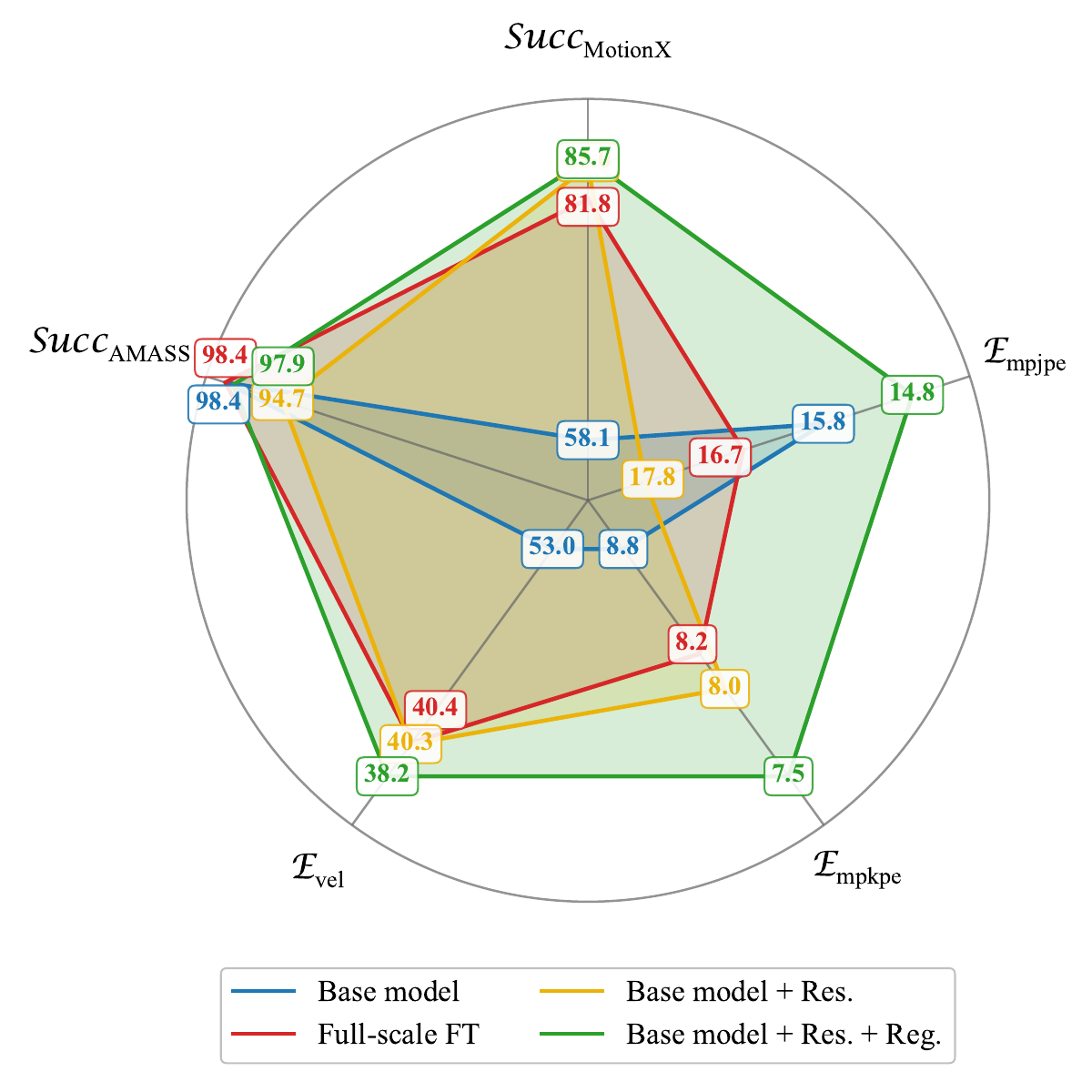}
\caption{Fast adaptation on LaFan1 and MotionX (target) with performance retention on AMASS (source).}
 \label{fig:adaptation_radar}
\end{figure}

On the target adaptation datasets, FAST consistently achieves the best overall performance across tracking accuracy and task completion metrics. Compared to training from scratch and full fine-tuning, FAST converges more rapidly and reaches higher success rates, demonstrating its effectiveness for fast adaptation to new motion distributions. Although the unregularized residual model also improves target-domain performance, it consistently underperforms FAST, indicating that structured regularization plays an important role in guiding stable and effective adaptation.

Importantly, FAST better preserves performance on the source dataset. As shown in Figure~\ref{fig:adaptation_radar}, the residual-only model exhibits a clear performance degradation on AMASS, reflecting substantial distribution drift caused by the lack of source-domain constraints. In contrast, FAST significantly alleviates this degradation, achieving the lowest MPJPE and MPKPE among all adapted models while maintaining competitive success rates. Although its success rate is slightly lower than that of the frozen base model and full fine-tuning, FAST provides a more favorable balance between adaptation to new data and robustness on the original distribution.

\subsection{Ablation Studies}
\subsubsection{Ablations on CoM-Aware Control}

To address \textbf{Q3}, we evaluate the impact of CoM-Aware Control on a curated set of high-dynamic motions, where maintaining balance is a primary source of failure. Table~\ref{tab:com} reports quantitative results on both in-distribution and out-of-distribution (OOD) subsets, while Figure~\ref{fig:com} presents representative qualitative examples.

\begin{table*}[h]
\centering
\setlength{\tabcolsep}{6pt}
\caption{Quantitative ablation results on CoM-Aware Control for high-dynamic motions.}
\label{tab:com}
\resizebox{\linewidth}{!}{%
\begin{tabular}{lccc @{\hskip 2em} ccc}
\toprule

& \multicolumn{3}{c}{\textbf{In-distribution dataset}} 
& \multicolumn{3}{c}{\textbf{Out-of-distribution dataset}} \\

\cmidrule(r){1-4} \cmidrule(r){5-7}

\multirow{2}{*}{\textbf{Method}} & \multicolumn{2}{c}{Stability} & \multicolumn{1}{c}{Completion} & \multicolumn{2}{c}{Stability} & \multicolumn{1}{c}{Completion} \\
\cmidrule(lr){2-3} \cmidrule(lr){4-4} \cmidrule(lr){5-6} \cmidrule(lr){7-7}
& $E_{\mathrm{mpd}}$ $\downarrow$ & $Slip$ $\downarrow$ & $Succ$ $\uparrow$ & $E_{\mathrm{mpd}}$ $\downarrow$ & $Slip$ $\downarrow$ & $Succ$ $\uparrow$ \\
\midrule
\method{} w.o. CoM       & 0.171$\pm$0.075    & 1.472$\pm$0.848                        & 84.15              & 0.122$\pm$0.052    & 1.298$\pm$0.752                        & 79.07                \\
\textbf{\method{} (Ours)} & \textbf{0.156$\pm$0.079} & \textbf{0.909$\pm$0.483}               & \textbf{97.56}        & \textbf{0.110$\pm$0.040} & \textbf{0.995$\pm$0.787}               & \textbf{86.63}            \\
\bottomrule
\end{tabular}
}
\end{table*}

As shown in Table~\ref{tab:com}, incorporating explicit CoM awareness consistently improves stability and task completion across both motion distributions. Compared to the variant without CoM-related objectives, the full model achieves lower mean positional deviation ($E_{\mathrm{mpd}}$) and reduced $Slip$, indicating more stable contact and improved balance maintenance. 

\begin{figure*}[h]
	\centering
        \includegraphics[width=\linewidth]{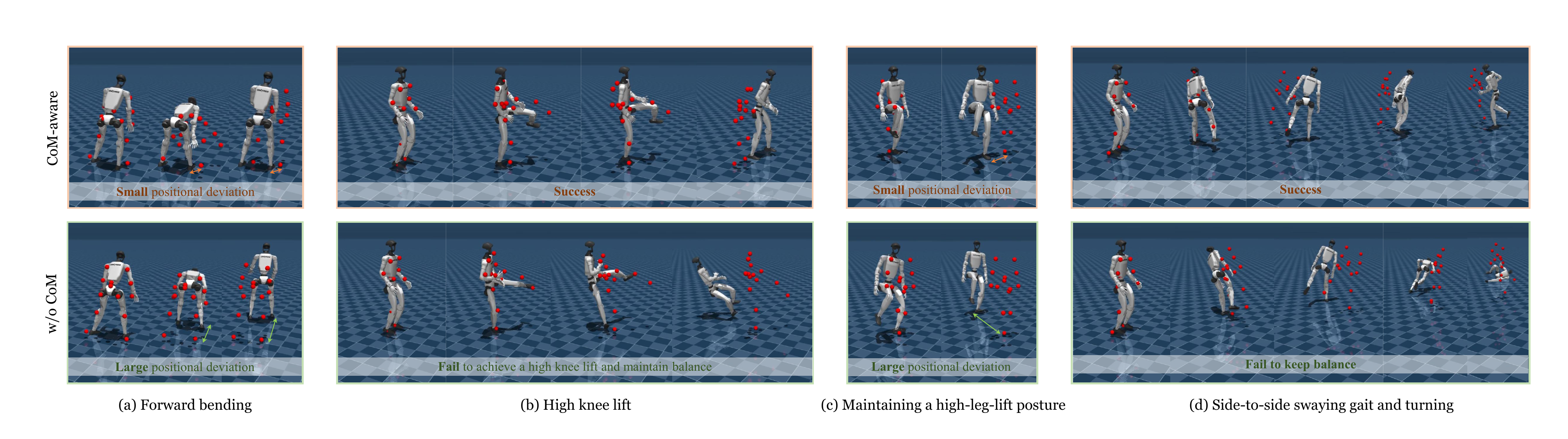}
\caption{Visualizations of representative motions comparing CoM-Aware Control with the baseline.
(a) Forward bending.
(b) High knee lift.
(c) High leg lift.
(d) Side-to-side swaying gait with turning.
CoM-Aware Control consistently maintains balance and reduces instability, while the baseline exhibits drift or falls.}
 \label{fig:com}
\end{figure*}

Figure~\ref{fig:com} further illustrates the benefits of CoM-Aware Control in representative high-dynamic scenarios. By explicitly accounting for the CoM, the controller suppresses excessive root drift and mitigates balance loss during large-amplitude motions, enabling successful execution where the baseline fails.

Overall, these results demonstrate that explicit consideration of CoM-related physical signals significantly enhances balance robustness in high-dynamic motion tracking and improves generalization to unseen motion distributions.

\subsubsection{Ablations on Fast Adaptation}

We conduct ablation studies by selectively removing Parseval and KL regularization from FAST, resulting in three variants: FAST w.o. Parseval \& KL, FAST w.o. KL, and FAST w.o. Parseval.

As shown in Table~\ref{tab:fast_adaptation}, on the target adaptation dataset, the FAST consistently achieves the best performance across all metrics. This indicates that Parseval and KL regularization provide complementary benefits for stable and efficient adaptation. Variants using only a single regularization term improve over the unregularized baseline but remain inferior to the full model.

\begin{table*}[t]
\centering
\caption{Abaltions on fast adaptation results on target adaptation dataset (LaFan1) and source dataset (AMASS). 
}
\resizebox{\textwidth}{!}{
\begin{tabular}{lcccc @{\hskip 0.8em} cccc}
\toprule
& \multicolumn{4}{c}{\textbf{LaFan1 (target adaptation dataset)}} 
& \multicolumn{4}{c}{\textbf{AMASS (source dataset)}} \\
\cmidrule(r){1-5} \cmidrule(r){6-9}
\textbf{Method} 
& $Succ \uparrow$ 
& $E_{\mathrm{mpjpe}} \downarrow$ 
& $E_{\mathrm{mpkpe}} \downarrow$ 
& $E_{\mathrm{vel}} \downarrow$ 
& $Succ \uparrow$ 
& $E_{\mathrm{mpjpe}} \downarrow$ 
& $E_{\mathrm{mpkpe}} \downarrow$ 
& $E_{\mathrm{vel}} \downarrow$ 
\\
\midrule
\method{} w.o. Parseval \& KL &
{86.19}  & 
{0.113}\textsubscript{$\pm$0.055}  & 
{0.067}\textsubscript{$\pm$0.055}  & 
{0.335}\textsubscript{$\pm$0.136}  & 
{94.67}  & 
{0.133}\textsubscript{$\pm$0.031}  & 
{0.050}\textsubscript{$\pm$0.026}  & 
{0.314}\textsubscript{$\pm$0.130}   
\\

\method{} w.o. KL &
{88.40}  & 
{0.090}\textsubscript{$\pm$0.052}  & 
{0.062}\textsubscript{$\pm$0.045}  & 
{0.318}\textsubscript{$\pm$0.121}  & 
{96.81}  & 
{0.099}\textsubscript{$\pm$0.027}  & 
{0.048}\textsubscript{$\pm$0.027}  & 
{0.297}\textsubscript{$\pm$0.130}  
\\

\method{} w.o. Parseval &
{87.82}  & 
{0.091}\textsubscript{$\pm$0.055}  & 
{0.064}\textsubscript{$\pm$0.068}  & 
{0.323}\textsubscript{$\pm$0.124}  & 
\textbf{97.36}  & 
{0.101}\textsubscript{$\pm$0.029}  & 
{0.048}\textsubscript{$\pm$0.027}  & 
\textbf{0.295}\textsubscript{$\pm$0.129}  
\\

\textbf{\method{} (ours)}    & 
\textbf{90.60}  & 
\textbf{0.084}\textsubscript{$\pm$0.050}  & 
\textbf{0.058}\textsubscript{$\pm$0.041}  & 
\textbf{0.317}\textsubscript{$\pm$0.123}  & 
{96.92}  & 
\textbf{0.092}\textsubscript{$\pm$0.028}  & 
\textbf{0.047}\textsubscript{$\pm$0.028}  & 
{0.306}\textsubscript{$\pm$0.137}   
\\

\bottomrule
\end{tabular}
}
\label{tab:fast_adaptation}
\end{table*}

On the source dataset, FAST and variant with KL regularization consistently achieve higher success rates than those without KL, indicating that KL regularization effectively helps preserve previously learned behaviors by constraining policy deviation from the base model. In particular, the KL-only variant attains the highest success rate, while FAST achieves the second-highest success rate together with the lowest MPJPE and MPKPE, reflecting a favorable balance between preservation and adaptation.

Overall, these results demonstrate that combining Parseval and KL regularization enables fast adaptation with a favorable balance between target-domain performance and source-domain preservation. 

\subsection{Real-world Deployment}
To answer \textbf{Q4}, we deploy FAST on a Unitree G1 humanoid robot \cite{unitree-g1} to evaluate its real-world performance across diverse motion sources and control scenarios, as shown in Figure 1. High-dynamic motion tracking and whole-body teleoperation are performed in a zero-shot manner using the pretrained general policy, demonstrating strong robustness and balance beyond simulation. For more challenging scenarios with low-quality motion references, we further leverage FAST’s lightweight residual adaptation to enable rapid and stable specialization.

\subsubsection{High-dynamic Motion Tracking}
We first evaluate FAST on high-dynamic single-motion tracking tasks using offline motion capture data. Without adaptation, the robot successfully executes challenging motions such as single-leg squats and side kicks~\cite{xie2025kungfubot}, which require precise balance control and rapid weight shifts. FAST maintains stable CoM behavior and accurately follows the reference motions, demonstrating strong zero-shot robustness to high-dynamic motions.

\subsubsection{Teleoperation}

We further assess FAST in real-time whole-body teleoperation scenarios, where human motions are streamed online as reference trajectories. The robot performs diverse behaviors, including walking, jumping, martial arts movements, and object manipulation (e.g., grabbing toys, and moving chairs). Despite the variability and latency inherent in teleoperation, FAST remains responsive and stable, illustrating that the pretrained general policy provides a strong foundation for whole-body teleoperation without additional adaptation.


\subsubsection{Adaptation to Low-Quality Motions}
\begin{wrapfigure}{r}{0.4\textwidth} 

    \centering
    \vspace{-15pt} 

    \includegraphics[width=\linewidth]{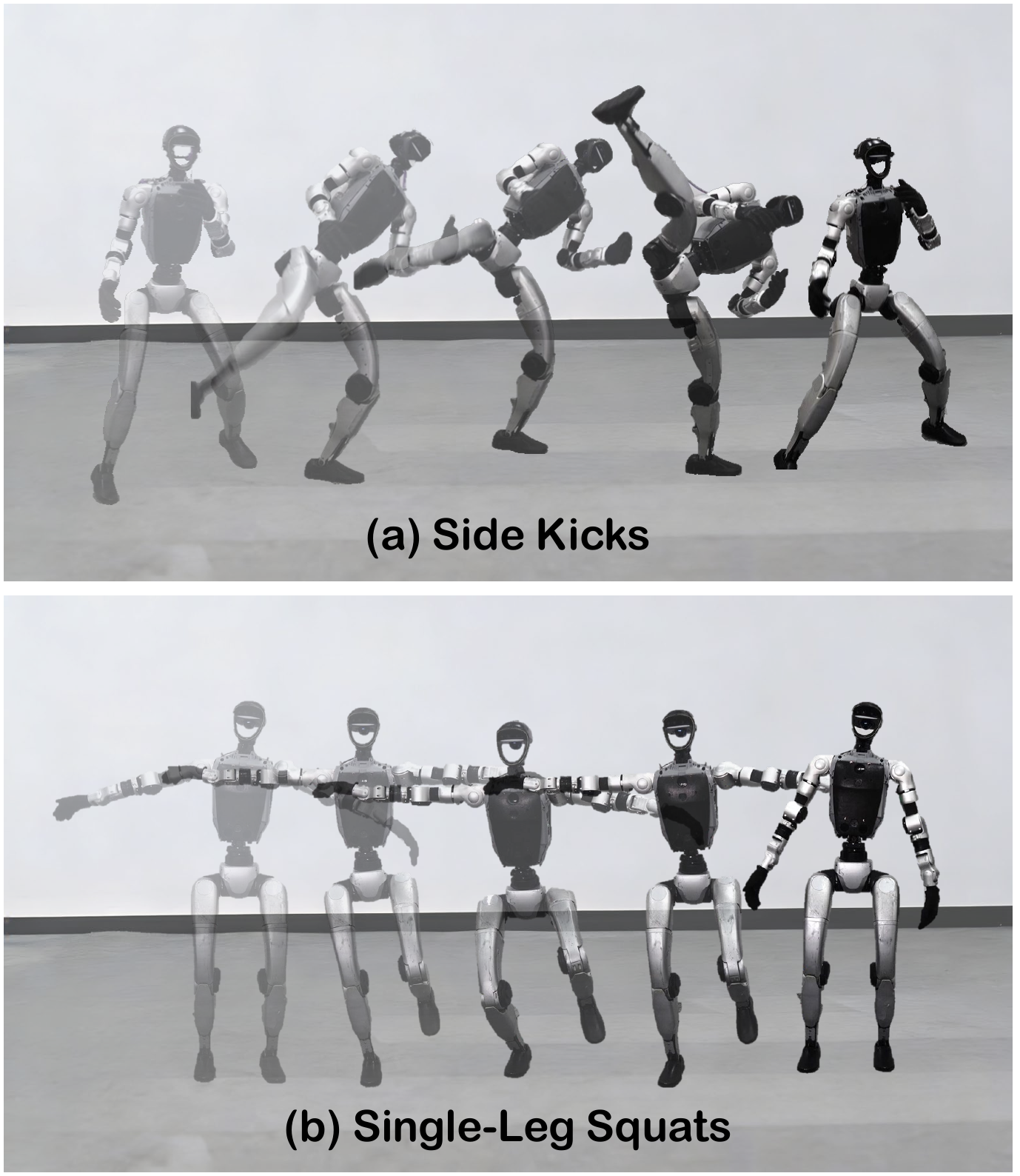}
    
    \caption{Zero-shot generalization to high-dynamic motions on a Unitree G1.}
    \label{fig:real_exp_high_dynamic}
    
    \vspace{-3em}
\end{wrapfigure}
While the pretrained base policy demonstrates strong zero-shot robustness, motion references derived from text generation or monocular video estimation often contain significant kinematic noise, temporal inconsistencies, and physical implausibilities that exceed the coverage of the training distribution. To handle such challenging scenarios, we leverage FAST’s residual adaptation mechanism to quickly specialize the base policy to these low quality motion sources.
 
We integrate a text-to-motion generation pipeline~\cite{yue2025rl} and a video-to-motion estimation pipeline~\cite{shen2024world,ze2025gmr} to produce reference motions. After fast adaptation, the robot robustly tracks synthesized motions, including single-leg jumping, 360-degree hopping, and complex dance sequences reconstructed from monocular videos. These results demonstrate that FAST enables efficient and stable adaptation to low-quality motions.


\section{Conclusion}

In this work, we present FAST, a general humanoid whole-body control framework that emphasizes fast adaptation and stable motion tracking. FAST combines Parseval-Guided Residual Policy Adaptation to efficiently specialize a pretrained policy to new motion distributions while constraining policy drift, with Center-of-Mass-Aware Control to enhance stability when tracking high-dynamic or low-quality reference motions. Extensive simulation and real-world experiments demonstrate that FAST consistently outperforms state-of-the-art methods in robustness, adaptability, and generalization across diverse scenarios, including motion capture tracking, text- and video-derived motions, and real-time teleoperation.
Future work will explore fully online and continual adaptation, as well as more adaptive regularization strategies to further balance stability and flexibility.

\clearpage

\bibliographystyle{unsrt}
\bibliography{ref}

@inproceedings{wang2025experts,
  title     = {From Experts to a Generalist: Toward General Whole-Body Control for Humanoid Robots},
  author    = {Yuxuan Wang and Ming Yang and Ziluo Ding and Yu Zhang and Weishuai Zeng and Xinrun Xu and Haobin Jiang and Zongqing Lu},
  booktitle = {The Thirty-ninth Annual Conference on Neural Information Processing Systems},
  year      = {2025}
}

@article{wang2025crosser,
  title={CROSSER: Learning Generalizable Humanoid Locomotion Through Inverse Dynamics-Guided Cross-Simulator Adaptation},
  author={Wang, Zepeng and Shi, Xiaochuan and Ding, Ziluo and Wang, Yuxuan and Sun, Zhenguo and Ma, Chao and Lu, Zongqing},
  journal={IEEE Robotics and Automation Letters},
  volume={10},
  number={12},
  pages={13336--13343},
  year={2025},
  publisher={IEEE}
}

@article{liao2025beyondmimic,
  title={Beyondmimic: From motion tracking to versatile humanoid control via guided diffusion},
  author={Liao, Qiayuan and Truong, Takara E and Huang, Xiaoyu and Tevet, Guy and Sreenath, Koushil and Liu, C Karen},
  journal={arXiv preprint arXiv:2508.08241},
  year={2025}
}

@article{chung2024parseval,
  title={Parseval regularization for continual reinforcement learning},
  author={Chung, Wesley and Cherif, Lynn and Meger, David and Precup, Doina},
  journal={Advances in Neural Information Processing Systems},
  volume={37},
  pages={127937--127967},
  year={2024}
}

@article{chen2025gmt,
  title={GMT: General Motion Tracking for Humanoid Whole-Body Control},
  author={Chen, Zixuan and Ji, Mazeyu and Cheng, Xuxin and Peng, Xuanbin and Peng, Xue Bin and Wang, Xiaolong},
  journal={arXiv preprint arXiv:2506.14770},
  year={2025}
}

@article{ze2025twist2,
  title={Twist2: Scalable, portable, and holistic humanoid data collection system},
  author={Ze, Yanjie and Zhao, Siheng and Wang, Weizhuo and Kanazawa, Angjoo and Duan, Rocky and Abbeel, Pieter and Shi, Guanya and Wu, Jiajun and Liu, C Karen},
  journal={arXiv preprint arXiv:2511.02832},
  year={2025}
}

@article{ze2025twist,
  title={Twist: Teleoperated whole-body imitation system},
  author={Ze, Yanjie and Chen, Zixuan and Ara{\'u}jo, Joao Pedro and Cao, Zi-ang and Peng, Xue Bin and Wu, Jiajun and Liu, C Karen},
  journal={arXiv preprint arXiv:2505.02833},
  year={2025}
}

@article{wang2025beamdojo,
  title={Beamdojo: Learning agile humanoid locomotion on sparse footholds},
  author={Wang, Huayi and Wang, Zirui and Ren, Junli and Ben, Qingwei and Huang, Tao and Zhang, Weinan and Pang, Jiangmiao},
  journal={arXiv preprint arXiv:2502.10363},
  year={2025}
}

@article{xie2025kungfubot,
  title={KungfuBot: Physics-Based Humanoid Whole-Body Control for Learning Highly-Dynamic Skills},
  author={Xie, Weiji and Han, Jinrui and Zheng, Jiakun and Li, Huanyu and Liu, Xinzhe and Shi, Jiyuan and Zhang, Weinan and Bai, Chenjia and Li, Xuelong},
  journal={arXiv preprint arXiv:2506.12851},
  year={2025}
}

@article{huang2025learning,
  title={Learning humanoid standing-up control across diverse postures},
  author={Huang, Tao and Ren, Junli and Wang, Huayi and Wang, Zirui and Ben, Qingwei and Wen, Muning and Chen, Xiao and Li, Jianan and Pang, Jiangmiao},
  journal={arXiv preprint arXiv:2502.08378},
  year={2025}
}

@article{ding2025jaeger,
  title   = {JAEGER: Dual-Level Humanoid Whole-Body Controller},
  author  = {Ding, Ziluo and Jiang, Haobin and Wang, Yuxuan and Sun, Zhenguo and Zhang, Yu and Niu, Xiaojie and Yang, Ming and Zeng, Weishuai and Xu, Xinrun and Lu, Zongqing},
  journal = {arXiv preprint arXiv:2505.06584},
  year    = {2025}
}

@article{ji2024exbody2,
  title={ExBody2: Advanced Expressive Humanoid Whole-Body Control}, 
  author={Ji, Mazeyu and Peng, Xuanbin and Liu, Fangchen and Li, Jialong and Yang, Ge and Cheng, Xuxin and Wang, Xiaolong},
  journal={arXiv preprint arXiv:2412.13196},
  year={2024},
}

@article{pan2025agility,
  title={Agility Meets Stability: Versatile Humanoid Control with Heterogeneous Data},
  author={Pan, Yixuan and Qiao, Ruoyi and Chen, Li and Chitta, Kashyap and Pan, Liang and Mai, Haoguang and Bu, Qingwen and Zhao, Hao and Zheng, Cunyuan and Luo, Ping and others},
  journal={arXiv preprint arXiv:2511.17373},
  year={2025}
}

@article{han2025kungfubot2,
  title={KungfuBot2: Learning Versatile Motion Skills for Humanoid Whole-Body Control},
  author={Han, Jinrui and Xie, Weiji and Zheng, Jiakun and Shi, Jiyuan and Zhang, Weinan and Xiao, Ting and Bai, Chenjia},
  journal={arXiv preprint arXiv:2509.16638},
  year={2025}
}

@article{luo2025sonic,
  title={Sonic: Supersizing motion tracking for natural humanoid whole-body control},
  author={Luo, Zhengyi and Yuan, Ye and Wang, Tingwu and Li, Chenran and Chen, Sirui and Casta{\~n}eda, Fernando and Cao, Zi-Ang and Li, Jiefeng and Minor, David and Ben, Qingwei and others},
  journal={arXiv preprint arXiv:2511.07820},
  year={2025}
}

@inproceedings{johannink2019residual,
  title={Residual reinforcement learning for robot control},
  author={Johannink, Tobias and Bahl, Shikhar and Nair, Ashvin and Luo, Jianlan and Kumar, Avinash and Loskyll, Matthias and Ojea, Juan Aparicio and Solowjow, Eugen and Levine, Sergey},
  booktitle={2019 international conference on robotics and automation (ICRA)},
  pages={6023--6029},
  year={2019},
  organization={IEEE}
}

@article{zhao2025resmimic,
  title={Resmimic: From general motion tracking to humanoid whole-body loco-manipulation via residual learning},
  author={Zhao, Siheng and Ze, Yanjie and Wang, Yue and Liu, C Karen and Abbeel, Pieter and Shi, Guanya and Duan, Rocky},
  journal={arXiv preprint arXiv:2510.05070},
  year={2025}
}

@article{he2025asap,
  title={Asap: Aligning simulation and real-world physics for learning agile humanoid whole-body skills},
  author={He, Tairan and Gao, Jiawei and Xiao, Wenli and Zhang, Yuanhang and Wang, Zi and Wang, Jiashun and Luo, Zhengyi and He, Guanqi and Sobanbab, Nikhil and Pan, Chaoyi and others},
  journal={arXiv preprint arXiv:2502.01143},
  year={2025}
}

@article{peng2022ase,
  title={Ase: Large-scale reusable adversarial skill embeddings for physically simulated characters},
  author={Peng, Xue Bin and Guo, Yunrong and Halper, Lina and Levine, Sergey and Fidler, Sanja},
  journal={ACM Transactions On Graphics (TOG)},
  volume={41},
  number={4},
  pages={1--17},
  year={2022},
  publisher={ACM New York, NY, USA}
}

@inproceedings{tessler2023calm,
  title={Calm: Conditional adversarial latent models for directable virtual characters},
  author={Tessler, Chen and Kasten, Yoni and Guo, Yunrong and Mannor, Shie and Chechik, Gal and Peng, Xue Bin},
  booktitle={ACM SIGGRAPH 2023 Conference Proceedings},
  pages={1--9},
  year={2023}
}

@inproceedings{pertsch2021accelerating,
  title={Accelerating reinforcement learning with learned skill priors},
  author={Pertsch, Karl and Lee, Youngwoon and Lim, Joseph},
  booktitle={Conference on robot learning},
  pages={188--204},
  year={2021},
  organization={PMLR}
}

@inproceedings{mees2024octo,
  title={Octo: An open-source generalist robot policy},
  author={Mees, Oier and Ghosh, Dibya and Pertsch, Karl and Black, Kevin and Walke, Homer Rich and Dasari, Sudeep and Hejna, Joey and Kreiman, Tobias and Xu, Charles and Luo, Jianlan and others},
  booktitle={First Workshop on Vision-Language Models for Navigation and Manipulation at ICRA 2024},
  year={2024}
}

@inproceedings{luo2023perpetual,
  title={Perpetual humanoid control for real-time simulated avatars},
  author={Luo, Zhengyi and Cao, Jinkun and Kitani, Kris and Xu, Weipeng and others},
  booktitle={Proceedings of the IEEE/CVF International Conference on Computer Vision},
  pages={10895--10904},
  year={2023}
}

@article{yin2025unitracker,
  title={Unitracker: Learning universal whole-body motion tracker for humanoid robots},
  author={Yin, Kangning and Zeng, Weishuai and Fan, Ke and Dai, Minyue and Wang, Zirui and Zhang, Qiang and Tian, Zheng and Wang, Jingbo and Pang, Jiangmiao and Zhang, Weinan},
  journal={arXiv preprint arXiv:2507.07356},
  year={2025}
}

@inproceedings{zhang2025hub,
  title={HuB: Learning Extreme Humanoid Balance},
  author={Zhang, Tong and Zheng, Boyuan and Nai, Ruiqian and Hu, Yingdong and Wang, Yen-Jen and Chen, Geng and Lin, Fanqi and Li, Jiongye and Hong, Chuye and Sreenath, Koushil and others},
  booktitle={RSS 2025 Workshop on Whole-body Control and Bimanual Manipulation: Applications in Humanoids and Beyond}
}

@inproceedings{mahmood2019amass,
  title={AMASS: Archive of motion capture as surface shapes},
  author={Mahmood, Naureen and Ghorbani, Nima and Troje, Nikolaus F and Pons-Moll, Gerard and Black, Michael J},
  booktitle={Proceedings of the IEEE/CVF international conference on computer vision},
  pages={5442--5451},
  year={2019}
}

@article{li2023object,
  title={Object motion guided human motion synthesis},
  author={Li, Jiaman and Wu, Jiajun and Liu, C Karen},
  journal={ACM Transactions on Graphics (TOG)},
  volume={42},
  number={6},
  pages={1--11},
  year={2023},
  publisher={ACM New York, NY, USA}
}

@article{lin2023motion,
  title={Motion-x: A large-scale 3d expressive whole-body human motion dataset},
  author={Lin, Jing and Zeng, Ailing and Lu, Shunlin and Cai, Yuanhao and Zhang, Ruimao and Wang, Haoqian and Zhang, Lei},
  journal={Advances in Neural Information Processing Systems},
  volume={36},
  pages={25268--25280},
  year={2023}
}

@article{zhang2025motion,
  title={Motion-x++: A large-scale multimodal 3d whole-body human motion dataset},
  author={Zhang, Yuhong and Lin, Jing and Zeng, Ailing and Wu, Guanlin and Lu, Shunlin and Fu, Yurong and Cai, Yuanhao and Zhang, Ruimao and Wang, Haoqian and Zhang, Lei},
  journal={arXiv preprint arXiv:2501.05098},
  year={2025}
}

@article{harvey2020robust,
  title={Robust motion in-betweening},
  author={Harvey, F{\'e}lix G and Yurick, Mike and Nowrouzezahrai, Derek and Pal, Christopher},
  journal={ACM Transactions on Graphics (TOG)},
  volume={39},
  number={4},
  pages={60--1},
  year={2020},
  publisher={ACM New York, NY, USA}
}

@article{yue2025rl,
  title={RL from Physical Feedback: Aligning Large Motion Models with Humanoid Control},
  author={Yue, Junpeng and Wang, Zepeng and Wang, Yuxuan and Zeng, Weishuai and Wang, Jiangxing and Xu, Xinrun and Zhang, Yu and Zheng, Sipeng and Ding, Ziluo and Lu, Zongqing},
  journal={arXiv preprint arXiv:2506.12769},
  year={2025}
}

@inproceedings{shen2024world,
  title={World-grounded human motion recovery via gravity-view coordinates},
  author={Shen, Zehong and Pi, Huaijin and Xia, Yan and Cen, Zhi and Peng, Sida and Hu, Zechen and Bao, Hujun and Hu, Ruizhen and Zhou, Xiaowei},
  booktitle={SIGGRAPH Asia 2024 Conference Papers},
  pages={1--11},
  year={2024}
}

@software{ze2025gmr,
  title={GMR: General Motion Retargeting},
  author= {Yanjie Ze and João Pedro Araújo and Jiajun Wu and C. Karen Liu},
  year= {2025},
  url= {https://github.com/YanjieZe/GMR},
  note= {GitHub repository}
}

@article{zhang2025track,
  title={Track any motions under any disturbances},
  author={Zhang, Zhikai and Guo, Jun and Chen, Chao and Wang, Jilong and Lin, Chenghuai and Lian, Yunrui and Xue, Han and Wang, Zhenrong and Liu, Maoqi and Lyu, Jiangran and others},
  journal={arXiv preprint arXiv:2509.13833},
  year={2025}
}

@inproceedings{he2025omnih2o,
  title={OmniH2O: Universal and Dexterous Human-to-Humanoid Whole-Body Teleoperation and Learning},
  author={He, Tairan and Luo, Zhengyi and He, Xialin and Xiao, Wenli and Zhang, Chong and Zhang, Weinan and Kitani, Kris M and Liu, Changliu and Shi, Guanya},
  booktitle={Conference on Robot Learning},
  pages={1516--1540},
  year={2025},
  organization={PMLR}
}

@inproceedings{he2025hover,
  title={Hover: Versatile neural whole-body controller for humanoid robots},
  author={He, Tairan and Xiao, Wenli and Lin, Toru and Luo, Zhengyi and Xu, Zhenjia and Jiang, Zhenyu and Kautz, Jan and Liu, Changliu and Shi, Guanya and Wang, Xiaolong and others},
  booktitle={2025 IEEE International Conference on Robotics and Automation (ICRA)},
  pages={9989--9996},
  year={2025},
  organization={IEEE}
}

@misc{unitree-g1,
  author = {{Unitree Robotics}},
  title = {{Humanoid robot G1\_Humanoid Robot Functions\_Humanoid Robot Price | Unitree Robotics}},
  year = {2025},
  url = {https://www.unitree.com/g1/},
  note = {\url{https://www.unitree.com/g1/}}
}

@article{he2025viral,
  title={VIRAL: Visual Sim-to-Real at Scale for Humanoid Loco-Manipulation},
  author={He, Tairan and Wang, Zi and Xue, Haoru and Ben, Qingwei and Luo, Zhengyi and Xiao, Wenli and Yuan, Ye and Da, Xingye and Casta{\~n}eda, Fernando and Sastry, Shankar and others},
  journal={arXiv preprint arXiv:2511.15200},
  year={2025}
}

@article{su2025hitter,
  title={Hitter: A humanoid table tennis robot via hierarchical planning and learning},
  author={Su, Zhi and Zhang, Bike and Rahmanian, Nima and Gao, Yuman and Liao, Qiayuan and Regan, Caitlin and Sreenath, Koushil and Sastry, S Shankar},
  journal={arXiv preprint arXiv:2508.21043},
  year={2025}
}

@article{li2025train,
  title={How to Train Your Robots? The Impact of Demonstration Modality on Imitation Learning},
  author={Li, Haozhuo and Cui, Yuchen and Sadigh, Dorsa},
  journal={arXiv preprint arXiv:2503.07017},
  year={2025}
}

@article{ahmad2026vision,
  title={Vision-Language Models on the Edge for Real-Time Robotic Perception},
  author={Ahmad, Sarat and Hafeez, Maryam and Zaidi, Syed Ali Raza},
  journal={arXiv preprint arXiv:2601.14921},
  year={2026}
}

@inproceedings{huang2010lcm,
  title={LCM: Lightweight communications and marshalling},
  author={Huang, Albert S and Olson, Edwin and Moore, David C},
  booktitle={2010 IEEE/RSJ International Conference on Intelligent Robots and Systems},
  pages={4057--4062},
  year={2010},
  organization={IEEE}
}

@article{schulman2017proximal,
  title={Proximal policy optimization algorithms},
  author={Schulman, John and Wolski, Filip and Dhariwal, Prafulla and Radford, Alec and Klimov, Oleg},
  journal={arXiv preprint arXiv:1707.06347},
  year={2017}
}

\clearpage

\beginappendix
\crefalias{section}{appendix}
\crefalias{subsection}{appendix}
\section{RL training details}
We provide a detailed training environment setting in this section.

\subsection{Observation} \label{training:obs}
We adopt a single-stage reinforcement learning framework and do not employ a teacher–student architecture \cite{ze2025twist,chen2025gmt,he2025omnih2o}.

\textbf{Critic Observations} (Privileged Information).
The critic is provided with privileged observations to facilitate stable value estimation. These include proprioception information: linear and angular velocities, joint positions and velocities, keypoint positions and rotations, root position and orientation (expressed in the local coordinate frame), and the last action. In addition, we augment the critic with physical signals, including contact masks, center-of-mass (CoM), and center-of-pressure (CoP). Task-related reference information is also included, consisting of target joint positions and velocities, target keypoint position differences relative to the root, target linear and angular velocities (all in the local coordinate frame), as well as target contact masks, CoM, and CoP signals.

\textbf{Actor Observations}.
The actor receives a restricted set of observations without privileged physical signals. Specifically, it observes linear and angular velocities, root position and orientation, joint positions and velocities, and the previous action. Task-related reference observations further include target joint positions and velocities, target selected keypoint (including wrists and ankles) position differences relative to the root, target linear and angular velocities, as well as reference CoM and CoP signals.

\textbf{Residual Policy Observations}.
The residual action policy uses the same observation space as the actor network to ensure consistency during fast adaptation.

\subsection{Reward} \label{training:reward}

We summarize the reward terms used for training both the base policy and the residual policy in Table \ref{tab:rewardterms}. To ensure consistency and generality, the same reward formulation is applied to both training stages.

\textbf{CoM–Aware Adaptive Tracking Weight.}
In addition to the explicit stability reward, we further modulate all motion-tracking rewards using an adaptive weight derived from the reference motion stability.
Let $d_\text{CoM-CoP}^\text{ref}$ denote the horizontal distance between the reference CoM and CoP. When this distance exceeds a critical threshold $\tau$, the reference motion is considered physically aggressive or unstable. In such cases, we progressively down-weight tracking-related rewards to grant the controller more freedom to maintain balance.

Concretely, we define an adaptive tracking weight
\begin{equation}
w_{\text{track}} =
\mathrm{clip}\left(
\exp\left(-\frac{\max(0, d^{\text{ref}}_{\text{CoM-CoP}} - \tau)}{\kappa}\right)
, w_{\min}, 1
\right),
\end{equation}
where $\tau$ is the stability threshold, $\kappa$ controls the decay rate, and $w_\text{min}$ is a small lower bound to preserve minimal tracking awareness. When $d_\text{CoM-CoP}^\text{ref}\leq\tau$ the weight remains close to 1, corresponding to full tracking. As the reference becomes increasingly unbalanced, the weight decays smoothly toward $w_\text{min}$.

This adaptive factor is multiplicatively applied to key tracking rewards, including body position, anchor position, and joint position terms:
\begin{equation}
r^{\prime}_{\text{track}} = w_{\text{track}} \cdot r_{\text{track}}.
\end{equation}
Such modulation enables a continuous trade-off between motion fidelity and physical stability, preventing the policy from over-committing to physically implausible references while retaining accurate tracking for stable motions.

\begin{table}[h]
\centering
\caption{Unified reward formulation using Gaussian-shaped tracking scores.}
\label{tab:rewardterms}
\begin{tabular}{llc}
\toprule
\textbf{Reward Terms} & \textbf{Expression} & \textbf{Weight} \\
\midrule
\multicolumn{3}{l}{\emph{\textbf{Task reward}}}\\
Joint position
& $\displaystyle
\exp\!\Big(
-\big( \tfrac{1}{|\mathcal{J}_{\mathrm{target}}|}
\sum_{b\in\mathcal{J}_{\mathrm{target}}}
\|\mathbf{q}^{\mathrm{des}}_{b}-\mathbf{q}_{b}\|^{2} \big) / 0.4^{2}
\Big)$
& $1.0$\\[2mm]
Body position
& $\displaystyle
\exp\!\Big(
-\big( \tfrac{1}{|\mathcal{B}_{\mathrm{target}}|}
\sum_{b\in\mathcal{B}_{\mathrm{target}}}
\|\mathbf{p}^{\mathrm{des}}_{b}-\mathbf{p}_{b}\|^{2} \big) / 0.3^{2}
\Big)$
& $1.0$\\[2mm]
Body orientation
& $\displaystyle
\exp\!\Big(
-\big( \tfrac{1}{|\mathcal{B}_{\mathrm{target}}|}
\sum_{b\in\mathcal{B}_{\mathrm{target}}}
\|\log(R^{\mathrm{des}}_{b}R_{b}^{\top})\|^{2} \big) / 0.4^{2}
\Big)$
& $1.0$\\[2mm]
Body linear velocity
& $\displaystyle
\exp\!\Big(
-\big( \tfrac{1}{|\mathcal{B}_{\mathrm{target}}|}
\sum_{b\in\mathcal{B}_{\mathrm{target}}}
\|\mathbf{v}^{\mathrm{des}}_{b}-\mathbf{v}_{b}\|^{2} \big) / 1.0^{2}
\Big)$
& $1.0$\\[2mm]
Body angular velocity
& $\displaystyle
\exp\!\Big(
-\big( \tfrac{1}{|\mathcal{B}_{\mathrm{target}}|}
\sum_{b\in\mathcal{B}_{\mathrm{target}}}
\|\boldsymbol{\omega}^{\mathrm{des}}_{b}-\boldsymbol{\omega}_{b}\|^{2} \big) / 3.14^{2}
\Big)$
& $1.0$\\[2mm]
Anchor position
& $\displaystyle
\exp\!\Big(
-\|\mathbf{p}^{\mathrm{des}}_{\text{anchor}}-\mathbf{p}_{\text{anchor}}\|^{2} / 0.3^{2}
\Big)$
& $0.5$\\[2mm]
Anchor orientation
& $\displaystyle
\exp\!\Big(
-\|\log(R^{\mathrm{des}}_{\text{anchor}}R_{\text{anchor}}^{\top})\|^{2} / 0.4^{2}
\Big)$
& $0.5$\\[2mm]
Contact mask
& $1-\|\mathbf{c}_{\text{robot}}-\mathbf{c}_{\text{ref}}\|_1/2$
& $1.0$\\[2mm]

\midrule
\multicolumn{3}{l}{\emph{\textbf{Regularization}}}\\
Action smoothness
& $\|\mathbf{a}_{t}-\mathbf{a}_{t-1}\|^{2}$
& $-0.1$\\[2mm]
Undesired self-contacts
& $\sum_{b\notin\mathcal{B}_{\mathrm{ee}}}
\mathbf{1}\!\left[\|f^{\mathrm{self}}_{b}\|>1  \text{N}\right]$
& $-0.1$\\[2mm]
Robot balance
& $\displaystyle
\exp\!\Big(
-\max(\|\mathbf{CoM}^{\mathrm{xy}}_{\text{robot}}-\mathbf{CoP}_{\text{xy}}^{xy}\|_{2} - 0.12, 0) / 0.08^{2}
\Big)$
& $-2.0$\\[2mm]

\midrule
\multicolumn{3}{l}{\emph{\textbf{Penalty}}}\\
Termination
& $\mathbf{1}\!\left[{\mathbf{termination}}\right]$
& $-1.0$\\[2mm]
Joint position limit
& $\sum_{j=1}^{N}\big[\max(l_j-\theta_j,0)+\max(\theta_j-u_j,0)\big]$
& $-10.0$\\[2mm]
\bottomrule
\end{tabular}
\end{table}

\subsection{Action}

We control the 29 degrees of freedom (DoF) of the Unitree G1 humanoid \cite{unitree-g1} using a proportional–derivative (PD) controller. The policy outputs target joint position offsets that are tracked by the PD controller.

Specifically, for each joint $j$, the target joint position at time step $t$ is computed as
\begin{equation}
    q_{j,t}=\hat{q}_j+\alpha_ja_{j,t},
\end{equation}
where $\hat{q}_j$ denotes a fixed default joint configuration, $a_{j,t}$ is the policy output, and $\alpha_j$ is a joint-specific scaling factor.

The scaling factor $\alpha_j$ is defined as
\begin{equation}
    \alpha_j=0.25\cdot\frac{\tau_{j,\mathrm{max}}}{k_{p,j}},
\end{equation}
where $\tau_{j,\mathrm{max}}$ is the maximum allowable torque for joint $j$, and $k_{p,j}$ is the proportional gain (stiffness) of the PD controller. This formulation ensures that the resulting joint position commands respect actuator torque limits.

The stiffness and damping parameters of the PD controller are set following prior work \cite{liao2025beyondmimic}, and are kept fixed throughout training and evaluation.

\subsection{Termination} \label{appendix:termination}
An episode is terminated and reset under one of the following three conditions:

\textbf{Fall Detection}. The episode is terminated if the humanoid is considered to have fallen. Specifically, a fall is detected if either
(i) the height of the robot root deviates from the reference motion by more than 0.25 m, or
(ii) the difference in the vertical gravity component between the reference motion’s anchor orientation and the robot’s anchor orientation exceeds 0.8 rad.

\textbf{Tracking Failure}. The episode is terminated if tracking fails, defined as any end-effector body (left or right ankle, or left or right hand), having a height deviation greater than 0.25 m from the corresponding reference motion.

\textbf{Task Completion}. If the humanoid successfully tracks the reference motion until the final frame without triggering any termination condition, the episode is considered successful and ends naturally.

\subsection{Domain randomization}

Detailed domain randomization setups are summarized in Table \ref{tab:domain_rand}, following \cite{liao2025beyondmimic}.

\begin{table}[h]
\centering
\caption{\textbf{Domain randomization parameters.} ($\mathcal{U}[\cdot]$: uniform distribution)}
\label{tab:domain_rand}
\begin{tabular}{ll}
\toprule
\textbf{Domain Randomization} & \textbf{Sampling Distribution} \\
\midrule
\multicolumn{2}{l}{\emph{\textbf{Physical parameters}}}\\
Static friction coefficients & $\mu_{\text{static}} \sim \mathcal{U}[0.3,\,1.6]$ \\[1mm]
Dynamic friction coefficients & $\mu_{\text{dynamic}} \sim \mathcal{U}[0.3,\,1.2]$ \\[1mm]
Restitution coefficient & $e_{\text{rest}} \sim \mathcal{U}[0,\,0.5]$ \\[1mm]
Default joint positions [rad] & $ \Delta\theta^0_j \!\sim  \mathcal{U}[-0.01,\,0.01] $\\[1mm]
Torso's CoM offset [m]& $\Delta x\!\sim\!\mathcal{U}[-0.025,0.025]$,  $\ \Delta y\!\sim\!\mathcal{U}[-0.05,0.05]$, $\  \Delta z\!\sim\!\mathcal{U}[-0.05,0.05]$ \\[1mm]
\midrule
\multicolumn{2}{l}{\emph{\textbf{Root velocity perturbations}}}\\[1mm]
Root linear vel [m/s]& $v_x\!\sim\!\mathcal{U}[-0.5,0.5],\ v_y\!\sim\!\mathcal{U}[-0.5,0.5],\ v_z\!\sim\!\mathcal{U}[-0.2,0.2]$ \\[1mm]
Push duration [s]& $\Delta t \sim \mathcal{U}[1.0,\,3.0]$ \\[1mm]
Root angular vel [rad/s]& $\omega_x,\ \omega_y\!\sim\!\mathcal{U}[-0.52,0.52],\ \omega_z\!\sim\!\mathcal{U}[-0.78,0.78]$ \\[1mm]
\bottomrule
\end{tabular}
\end{table}

\subsection{Adpative sampling}

To improve learning efficiency on complex and long-horizon motions, we introduce a two-stage adaptive sampling strategy that dynamically resamples training motions based on failure statistics collected during policy learning. The first stage operates at the clip level, where each target motion assigned to an environment is uniformly segmented into fixed-length clips. During training, we record the clip index corresponding to each episode termination or failure. At regular intervals, we estimate a failure rate for each clip and update a per-motion clip sampling distribution accordingly. Environments are then reinitialized by resampling the starting clip based on this distribution, biasing training toward temporally localized motion segments that are more difficult to track. This clip-level resampling mechanism enables the policy to focus on challenging sub-sequences within long motions, significantly improving sample efficiency for learning complex transitions, following a strategy similar to \cite{liao2025beyondmimic}.

The second stage operates at the motion level, targeting difficulty imbalance across the entire motion dataset. We aggregate failure statistics for each motion across all environments and periodically update a global motion sampling distribution proportional to their observed failure rates. At resampling steps, all environments are reassigned new target motions drawn from this distribution, increasing the training frequency of motions that are globally harder to execute. By jointly applying clip-level and motion-level adaptive resampling, the proposed two-stage strategy balances fine-grained temporal difficulty with coarse-grained motion complexity, enabling more effective learning of diverse and high-difficulty whole-body behaviors.

\subsection{Model details}

As described in the main paper, following prior work \cite{chen2025gmt,wang2025experts}, we adopt a mixture-of-experts (MoE) architecture to model diverse humanoid motion dynamics. Both the policy (actor) and value (critic) networks consist of multiple expert networks and a shared gating network.

\begin{table}[h]
\centering
\caption{\textbf{Hyperparameters for policy training.}}
\label{tab:hyperparameters}
\begin{tabular}{ll}
\toprule
\textbf{Term} & \textbf{Value} \\
\midrule
Num parallel envs & 4096 \\[1mm]
Num steps per env & 24 \\[1mm]
Learning epochs & 5 \\[1mm]
Num mini\_batches & 4 \\[1mm]
Discount $\gamma$ & 0.99 \\[1mm]
GAE $\lambda$ & 0.95 \\[1mm]
PPO clip ratio & 0.2 \\[1mm]
Value loss coefficient & 1.0 \\[1mm]
Entropy coefficient & 1e-2 \\[1mm]
Initial learning rate & 1e-3 \\[1mm]
Max learning rate & 1e-2 \\[1mm]
Min learning rate & 1e-5 \\[1mm]
Desired KL & 0.01 \\[1mm]
Max grad norm & 1.0 \\[1mm]
MLP hidden layers & [512, 256, 128] \\[1mm]
Num experts & 12 \\[1mm]
Parseval scaling factor & 2.0 \\[1mm]
Parseval loss coefficient & 1e-7 \\[1mm]
KL regularization coefficient & 1e-4 \\[1mm]
\bottomrule
\end{tabular}
\end{table}

Each expert network is implemented as a three-layer multilayer perceptron (MLP) with hidden layer sizes of 512, 256, and 128, respectively. The gating network takes the same observation as input and outputs state-dependent mixture weights over experts. The hidden layer dimension of gating network is set to half of the corresponding input dimension for both the actor and critic. At each timestep, the final policy action (or value estimate) is computed as a weighted sum of the outputs from all experts.

For fast adaptation, the residual (delta) policy is parameterized by a lightweight three-layer MLP with the same hidden layer sizes (512, 256, 128), but without an MoE structure. This design reduces optimization complexity while retaining sufficient expressivity for adaptation. All networks use the ELU activation function.

Other hyperparameters related to policy training are shown in Table \ref{tab:hyperparameters}.

\section{Proofs of Stability Analysis of Residual Policy Adaptation} \label{prof:prop}

\noindent\textit{Proof of Proposition~\ref{prop:Lipschitz} (Lipschitz Continuity with Scaled Orthogonality).} The residual policy $\pi_r$ is a composition of affine transformations and activation functions:
\begin{equation}
    \pi_r = f_L \circ \phi \circ f_{L-1} \circ \cdots \circ \phi \circ f_1,
\quad f_l(x)=W_l x + b_l.
\end{equation}

Since the activation function $\phi(\cdot)$ is assumed to be 1-Lipschitz, it satisfies
\begin{equation}
    \|\phi(x)-\phi(y)\|_2 \le \|x-y\|_2,\quad \forall x,y.
\end{equation}

Each affine transformation $f_l$ is Lipschitz continuous with constant $\|W_l\|_2$:
\begin{equation}
    \|f_l(x)-f_l(y)\|_2 = \|W_l(x-y)\|_2 \le \|W_l\|_2\|x-y\|_2,
\end{equation}
where the bias terms cancel out and thus do not affect the Lipschitz constant.

By the composition property of Lipschitz functions,
\begin{equation}
    \mathrm{Lip}(\pi_r) \le \prod_{l=1}^L \|W_l\|_2.
\end{equation}

Under Parseval regularization, we assume that for each layer $l=1,\ldots,L-1$,
\begin{equation}
    \|W_l^\top W_l - sI\|_2 \le \varepsilon_l.
\end{equation}

Since $\|W_l\|_2^2 = \lambda_{\max}(W_l^\top W_l)$, it follows that
\begin{equation}
    \|W_l\|_2^2
    = \lambda_{\max}(W_l^\top W_l)
    \le \lambda_{\max}(sI) + \|W_l^\top W_l - sI\|_2
    \le s + \varepsilon_l,
\end{equation}
and hence
\begin{equation}
    \|W_l\|_2 \le \sqrt{s + \varepsilon_l}.
\end{equation}

Substituting this bound yields
\begin{equation}
    \mathrm{Lip}(\pi_r)
    \le \left(\prod_{l=1}^{L-1}\sqrt{s + \varepsilon_l}\right)\|W_L\|_2.
\end{equation}

Thus, the residual policy $\pi_r$ is Lipschitz continuous with a bounded Lipschitz constant.
\hfill$\square$


\noindent\textit{Proof of Proposition~\ref{prop:stability} (Distributional Stability under KL Constraint).} For Gaussian policies with identical covariance $\Sigma$, the KL divergence reduces to
\begin{equation}
    D_{\mathrm{KL}}(\pi\|\pi_b)
    = \frac{1}{2}(\mu-\mu_b)^\top \Sigma^{-1}(\mu-\mu_b).
\end{equation}
Applying the KL constraint yields
\begin{equation}
    (\mu-\mu_b)^\top \Sigma^{-1}(\mu-\mu_b) \le 2\epsilon.
\end{equation}
By the Rayleigh quotient bound,
\begin{equation}
    (\mu-\mu_b)^\top \Sigma^{-1}(\mu-\mu_b)
    \ge \lambda_{\min}(\Sigma^{-1}) \|\mu-\mu_b\|_2^2
    = \frac{1}{\lambda_{\max}(\Sigma)} \|\mu-\mu_b\|_2^2.
\end{equation}
Combining the above inequalities gives
\begin{equation}
    \|\mu-\mu_b\|_2^2 \le 2\epsilon \lambda_{\max}(\Sigma).
\end{equation}
Recalling that $\mu=\mu_b+\mu_r$ holds pointwise for each state, we have $\mu_r=\mu-\mu_b$, and thus
\begin{equation}
    \|\mu_r\|_2 \le \sqrt{2\epsilon\,\lambda_{\max}(\Sigma)},
\end{equation}
which directly bounds the magnitude of the residual policy output.
\hfill$\square$

\section{Dataset description} \label{appendix:dataset}
\subsection{Dataset composition}
As described in the \ref{exp:dataset}, we employ multiple datasets for training and evaluation to comprehensively assess the generality, robustness, and scalability of FAST. In addition to commonly used motion datasets such as AMASS, OMOMO, and LaFan1, we further incorporate Motion-X and an in-house motion capture dataset. These datasets differ in motion quality, diversity, and acquisition modality, and are used for complementary training and evaluation purposes, as detailed below.

\textbf{Motion-X Dataset.} Motion-X is a large-scale 3D expressive full-body motion dataset constructed from a combination of large amounts of online videos and eight existing motion datasets. Compared to optically captured datasets such as AMASS, motions in Motion-X, especially those extracted from videos, exhibit lower geometric accuracy, but offer substantially greater diversity in motion types and scene coverage \cite{zhang2025motion}.

We use a leg-dominant subset of the Motion-X dataset for both base model evaluation and fast adaptation. Motion-X is selected due to its relatively low-quality and noisy motion signals, which provide a challenging benchmark for assessing robustness to imperfect reference motions. All Motion-X data are processed using the same retargeting pipeline described in Section~\ref{method:retarget}.

To emphasize high-dynamic behaviors, we cluster the Motion-X dataset based on text annotations and extract motion clips dominated by leg-intensive actions. The resulting subset contains 6,009 clips (approximately 15.24 hours of motion data) and is consistently used as both the test set for base model comparison and the target dataset for fast adaptation experiments.

\textbf{In-House Motion Capture Dataset.} In addition to public datasets, we collect an in-house indoor motion capture dataset covering a broad range of full-body motions. This dataset includes both common locomotion and manipulation behaviors, as well as high-dynamic motions, and serves as an additional source of real-world motion data.

From this dataset, we further identify and curate a subset of high-dynamic motions specifically for evaluating the effectiveness of Center-of-Mass-Aware Control. This subset includes actions such as single-leg balancing, kicking, and punching, with each motion sequence lasting at least 5 seconds. From the curated high-dynamic subset, we randomly select 12.6 minutes of data as high-dynamic motions included in the training set, and reserve 51.7 minutes as an out-of-distribution (OOD) zero-shot evaluation set. The composition of the OOD high-dynamic dataset is illustrated in Figure. \ref{fig:mocap_dynamic}.

\begin{figure}[h]
	\centering
        \includegraphics[width=0.5\linewidth]{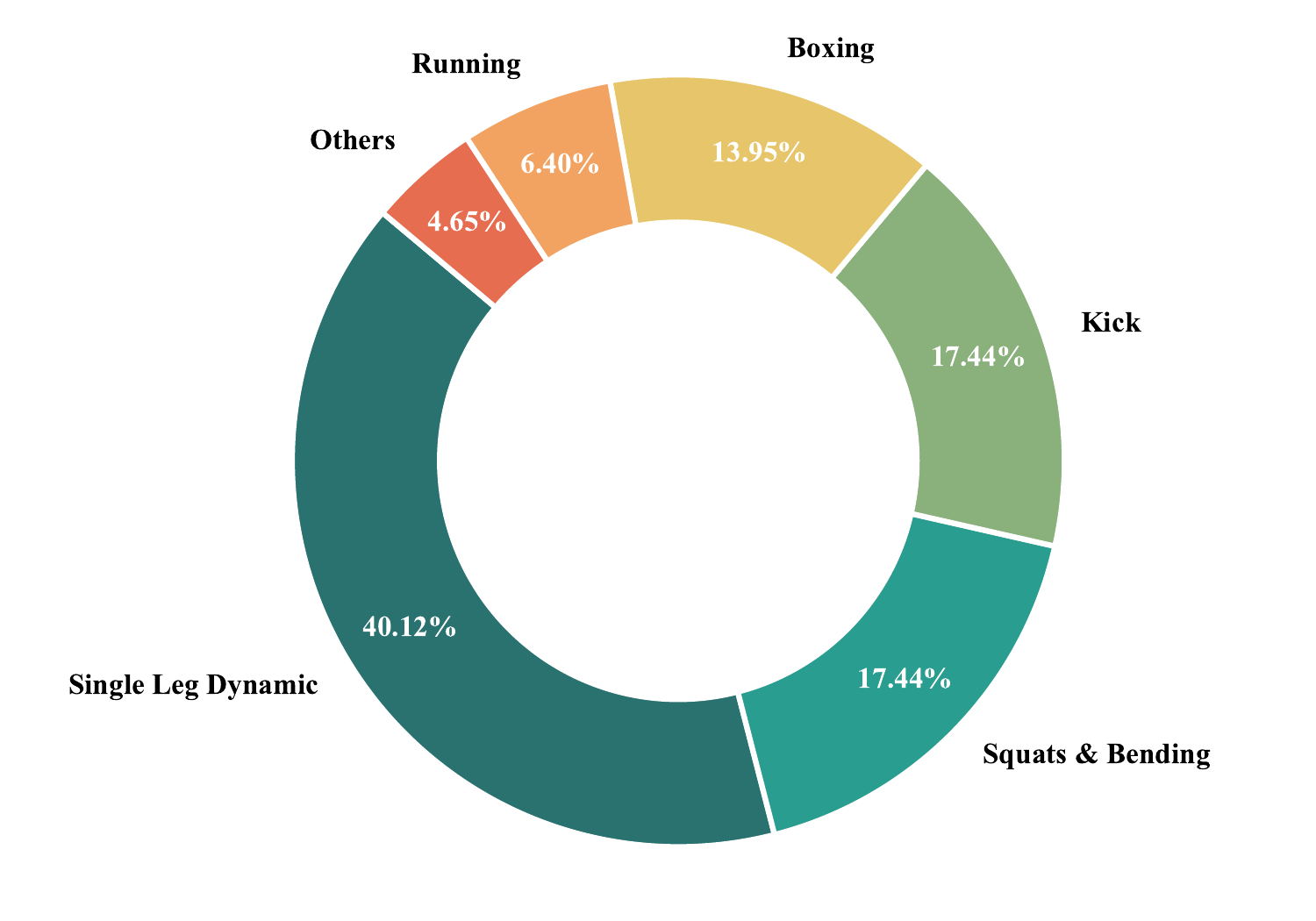}
\caption{The composition of the OOD high-dynamic dataset.}
 \label{fig:mocap_dynamic}
\end{figure}

\subsection{Dataset preprocessing} \label{data_augmentation}
The retargeted motions are further processed through a series of steps before being used for policy training. These steps aim to improve motion diversity, physical consistency, and the availability of auxiliary training signals.

\textbf{Motion Augmentation via Physical Attribute Perturbation.} To increase motion diversity and improve robustness to distribution shifts, we apply targeted data augmentation directly to the retargeted humanoid motions by modifying selected physical attributes. Specifically, we adjust the global root translation and velocity to generate motions with different execution speeds.

In addition, we perturb knee joint angles to synthesize varied walking postures while preserving overall motion coherence. To compensate for these knee modifications, the pitch angles of the hip and ankle joints are adjusted by half of the knee angle change. This heuristic maintains kinematic consistency and avoids introducing unrealistic poses. These controlled perturbations effectively expand the motion distribution without violating basic physical plausibility, encouraging the policy to generalize across a wider range of motion styles and speeds.

\textbf{Global Height Adjustment for Physical Consistency.} Due to retargeting and the aforementioned pose perturbations, the global vertical translation of the motion may no longer satisfy basic physical constraints, such as maintaining foot contact above the ground plane. To address this issue, we perform a global height correction on each motion sequence.

Specifically, we estimate a reasonable minimum foot height using the lower quartile (25th percentile) of foot heights over the entire motion sequence. Based on this estimate, a vertical offset is applied uniformly to all frames in the sequence to ensure that the feet remain physically plausible with respect to the ground. This procedure preserves relative motion structure while restoring global consistency.

\textbf{Contact Mask Estimation.} Following height adjustment, we compute contact masks based on the distance between the feet and the ground. The contact threshold is determined empirically using MuJoCo visualization, where foot–ground contact can be visually inspected through rendered contact indicators.

We note that different motion batches may require different contact thresholds due to variations in retargeting quality, physical constraint satisfaction, and motion capture accuracy. Therefore, the threshold is adjusted on a per-batch basis to ensure reliable contact estimation.

\textbf{Motion Replay and Auxiliary Signal Extraction.} Finally, we use IsaacLab as a forward dynamics engine to replay the processed motion sequences. By combining motion replay with finite-difference approximations, we extract additional auxiliary signals required for training, including body positions, body orientations, linear velocities, angular velocities, center-of-mass (CoM), center-of-pressure (CoP), joint positions and joint velocities.

The resulting motion data thus contains rich kinematic and physical annotations, which are used to support stable and effective training of the general whole-body controller.

\section{Experiment details}
\subsection{Task Failure conditions} \label{evaluation_failure}

The task failure conditions vary across experiments to reflect different evaluation objectives.

For \textbf{fast adaptation experiments}, the failure conditions are identical to those described in Appendix \ref{appendix:termination}, ensuring consistency with the base training setup.

For \textbf{base model comparison}, to better highlight differences in global position tracking accuracy, we introduce an additional failure criterion: an episode is terminated if the robot’s root position deviates from the reference motion by more than 2.0 m. This criterion emphasizes long-horizon global tracking performance.

For \textbf{CoM-Aware Control experiments on high-dynamic motions}, we further tighten the global root position deviation threshold to 1.5 m. This stricter criterion more sensitively captures failures caused by balance loss that lead to excessive global drift. At the same time, we relax the height deviation constraints on end-effectors, so that the success metric better reflects balance-related performance rather than strict limb-level tracking accuracy.

Moreover, in the latter two experimental settings, we do not introduce an explicit orientation-based termination criterion. This is because large orientation instabilities are already implicitly captured by the global root position deviation condition, which reflects physically meaningful failures such as loss of balance or unrecoverable drift. Moreover, for the CoM-aware control evaluation, introducing additional rotation-based termination may prematurely terminate relatively conservative but stable behaviors that involve transient torso tilting, thereby shifting the focus away from the core objective of balance robustness. By relying on global position drift as the primary failure signal, we better assess the controller’s ability to maintain physically stable and recoverable motion under challenging dynamics.

\subsection{Deployment}
The control policy is executed at an inference frequency of 50 Hz, while the low-level control interface operates at 500 Hz to ensure smooth and responsive actuation. Communication between the high-level policy and the low-level controller is handled via Lightweight Communications and Marshalling (LCM) \cite{huang2010lcm}, enabling efficient real-time data exchange.

State estimation, including linear velocity and position estimation, is implemented using the legged\_control2 library \cite{liao2025beyondmimic}. To support fast and stable inference during real-world deployment, the trained policy is exported to the ONNX format and executed using an optimized runtime.

All policy-related modules, including the neural policy and the state estimator, are deployed directly on the robot computer. Motion capture acquisition and retargeting are executed on an external workstation (Inter Core Ultra 7 255HX, 20 cores). The two machines communicate wirelessly via Wi-Fi using ZeroMQ, maintaining a stable bidirectional motion streaming rate of 50 Hz for reference transmission and feedback synchronization.

\clearpage

\end{document}